\documentclass[lettersize,journal]{IEEEtran}
\usepackage{amsmath,amsfonts}
\usepackage{mathtools}
\usepackage{algorithm}
\usepackage{multirow}
\usepackage{algpseudocode} 
\algrenewcommand\algorithmicindent{1.0em}
\let\STATE\State
\let\IF\If
\let\ENDIF\EndIf

\let\FOR\For
\let\ENDFOR\EndFor

\algdef{SE}[SERVER]{Server}{EndServer}{\textbf{At the server:}}{}
\algtext*{EndServer} 

\algdef{SE}[Ge]{Ge}{EndGe}{\textbf{Gradient estimation:}}{}
\algtext*{EndGe} 

\usepackage{array}

\usepackage{cite}
\usepackage{textcomp}
\usepackage{stfloats}
\usepackage{url}
\usepackage{verbatim}
\usepackage{graphicx}
\usepackage{subcaption}
\usepackage{caption}
\def\BibTeX{{\rm B\kern-.05em{\sc i\kern-.025em b}\kern-.08em
    T\kern-.1667em\lower.7ex\hbox{E}\kern-.125emX}}
\usepackage{enumerate}
\usepackage{amsthm}
\usepackage{dsfont}

\usepackage{float}
\usepackage{mdwtab}
\usepackage{makecell}
\usepackage{amssymb}
\usepackage{amsfonts}
\usepackage{enumitem}
\usepackage{color}
\usepackage{wrapfig}
\usepackage{hyperref}
\usepackage{enumerate}
\usepackage{booktabs}

\usepackage{xcolor} 

\newtheorem{theorem}{Theorem}

\newtheorem{Lemma}{Lemma}
\newtheorem{Remark}{Remark}

\newtheorem{Assumption}{Assumption}

\usepackage{balance}

\usepackage{bm}
\def \tw{\tilde{\bm{w}}}
\def \tbw{\tilde{\bar{\bm{w}}}}

\def \S{\mathcal{S}}

\def \E{\mathbb{E}}

\def \tg{\tilde{\bm{g}}}

\definecolor{fedblue}{RGB}{85,142,251}
\definecolor{fedgreen}{RGB}{87,171,90}
\definecolor{fedorange}{RGB}{255,153,128}

\algdef{SE}[SERVER]{Server}{EndServer}{\textbf{At the server:}}{}
\algtext*{EndServer} 

\algdef{SE}[Ge]{Ge}{EndGe}{\textbf{Gradient estimation:}}{}
\algtext*{EndGe} 

\title{Gradient Compression May Hurt Generalization: A Remedy by Synthetic Data Guided Sharpness Aware Minimization}

\author{%
    Yujie Gu, Richeng Jin, \textit{Member, IEEE}, Zhaoyang Zhang, \textit{Senior Member, IEEE}, Huaiyu Dai, \textit{Fellow, IEEE}%
    \thanks{Yujie Gu, Richeng Jin, and Zhaoyang Zhang are with the Department of Information and Communication Engineering, Zhejiang University, Hangzhou 310007, China, and also with Zhejiang Provincial Key Laboratory of Multi-Modal Communication Networks and Intelligent Information Processing, Hangzhou 310007, China (e-mail: \{yujie\_gu, richengjin, ning\_ming\}@zju.edu.cn).}%
    \thanks{Huaiyu Dai is with the Department of Electrical and Computer Engineering, North Carolina State University, Raleigh, NC 27695, USA. (e-mail: hdai@ncsu.edu).}%
}

\begin{document}
\maketitle

\begin{abstract} It is commonly believed that gradient compression in federated learning (FL) enjoys significant improvement in communication efficiency with negligible performance degradation. In this paper, we find that gradient compression induces sharper loss landscapes in federated learning, particularly under non-IID data distributions, which suggests hindered generalization capability. The recently emerging Sharpness Aware Minimization (SAM) effectively searches for a flat minima by incorporating a gradient ascent step (i.e., perturbing the model with gradients) before the celebrated stochastic gradient descent. Nonetheless, the direct application of SAM in FL suffers from inaccurate estimation of the global perturbation due to data heterogeneity. Existing approaches propose to utilize the model update from the previous communication round as a rough estimate. However, its effectiveness is hindered when model update compression is incorporated. In this paper, we propose FedSynSAM, which leverages the global model trajectory to construct synthetic data and facilitates an accurate estimation of the global perturbation. The convergence of the proposed algorithm is established, and extensive experiments are conducted to validate its effectiveness. 
\end{abstract}


\section{Introduction}
\label{introduction}
\noindent The explosion of data from mobile and edge devices has driven interest in distributed machine learning approaches. Federated learning (FL) has emerged as a promising paradigm where clients collaborate to train models while keeping the training data local. As neural networks grow increasingly large, communication efficiency has become a critical bottleneck in FL.
To address the challenge, FedAvg \cite{mcmahan2017communication}, which enables less frequent information exchange by performing multiple local model updates at clients, has been widely adopted. Another line of research proposes to incorporate gradient compression, e.g., stochastic quantization \cite{alistairh2017qsgd} and Top-$k$ sparsification \cite{alistairh2018convergence} to reduce the size of exchanged messages.

Numerous existing works have shown that SGD-based algorithms with gradient compression attain a comparable convergence rate to their full-precision counterparts from a theoretical perspective. Therefore, it is believed that gradient compression brings the benefit of significant improvement in communication efficiency \cite{reisizadeh2020fedpaq,alistairh2018convergence,Sattler2020TNNLS} with a negligible performance degradation. Nonetheless, we find that gradient compression may lead to sharper loss landscapes, suggesting that the models tend to suffer from degradation in generalization performance. As visualized in Figure \ref{fig:fig_pre} and Table \ref{tab:hessian_eigenvalues}, the vanilla FedAvg enjoys a flatter loss landscape compared to the counterparts with gradient compression, regardless of compression ratios and data distribution. In addition, more aggressive compression and more severe data heterogeneity tend to result in a sharper loss landscape.

\begin{figure}[t]
    \centering
    \begin{subfigure}[b]{0.24\textwidth}
        \centering
        \includegraphics[width=1\textwidth, trim=1.2cm 1.2cm 1.4cm 1.4cm, clip]{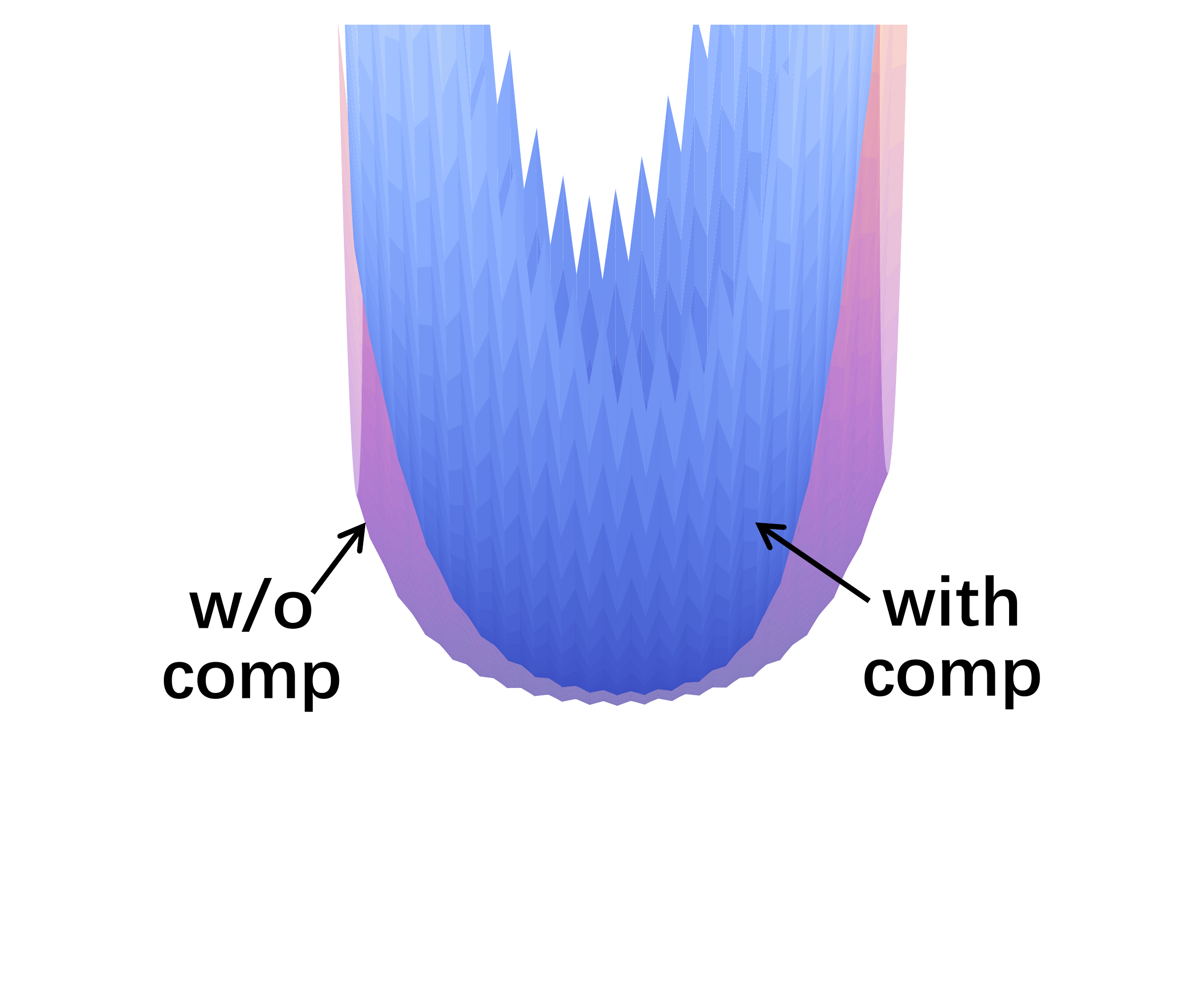}
        \vspace{-0.5in}
        \caption{4 bit Quantization,\\IID}
        \label{fig:image1}
    \end{subfigure}
    \hfill
    \begin{subfigure}[b]{0.24\textwidth}
        \centering
        \includegraphics[width=1\textwidth, trim=1.2cm 1.2cm 1.4cm 1.4cm, clip]{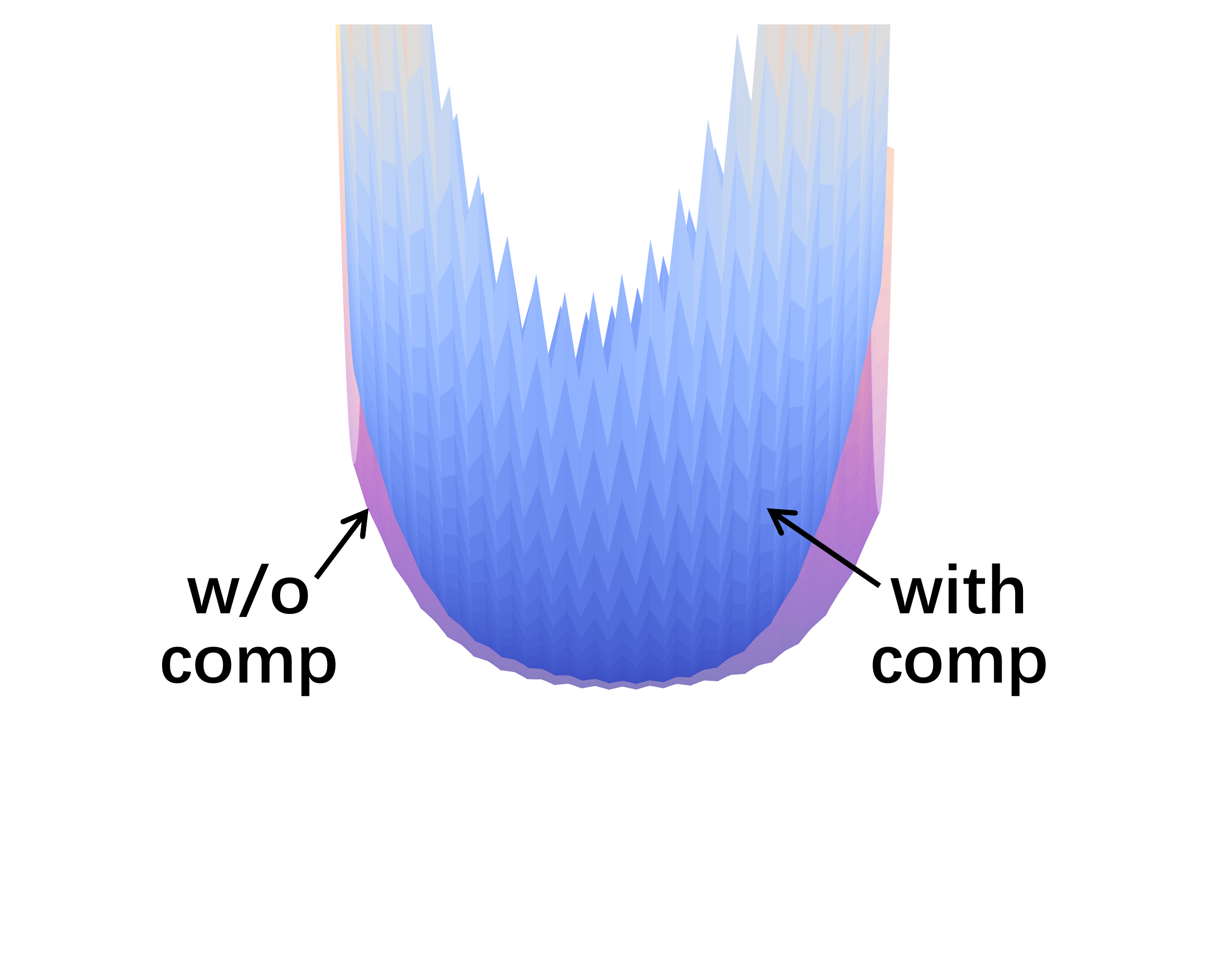}
        \vspace{-0.5in}
        \caption{Top-$k$ 0.25 sparsity,\\IID}
        \label{fig:image2}
    \end{subfigure}
    \hfill
    \begin{subfigure}[b]{0.24\textwidth}
        \centering
        \includegraphics[width=1\textwidth, trim=1.2cm 1.2cm 1.4cm 1.4cm, clip]{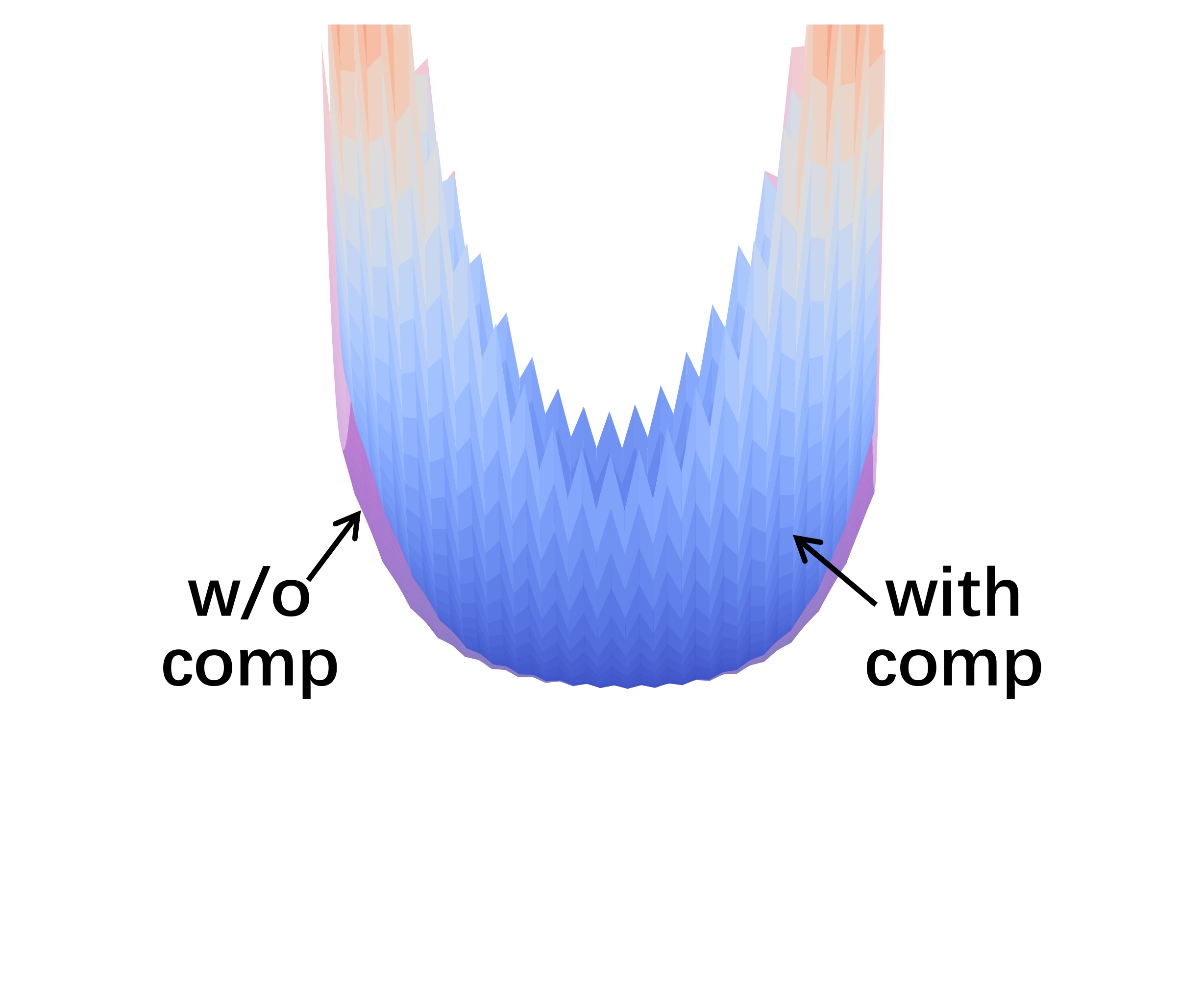}
        \vspace{-0.5in}
        \caption{8 bit Quantization,\\IID}
        \label{fig:image3}
    \end{subfigure}
    \hfill
    \begin{subfigure}[b]{0.24\textwidth}
        \centering
        \includegraphics[width=1\textwidth, trim=1.2cm 1.2cm 1.4cm 1.4cm, clip]{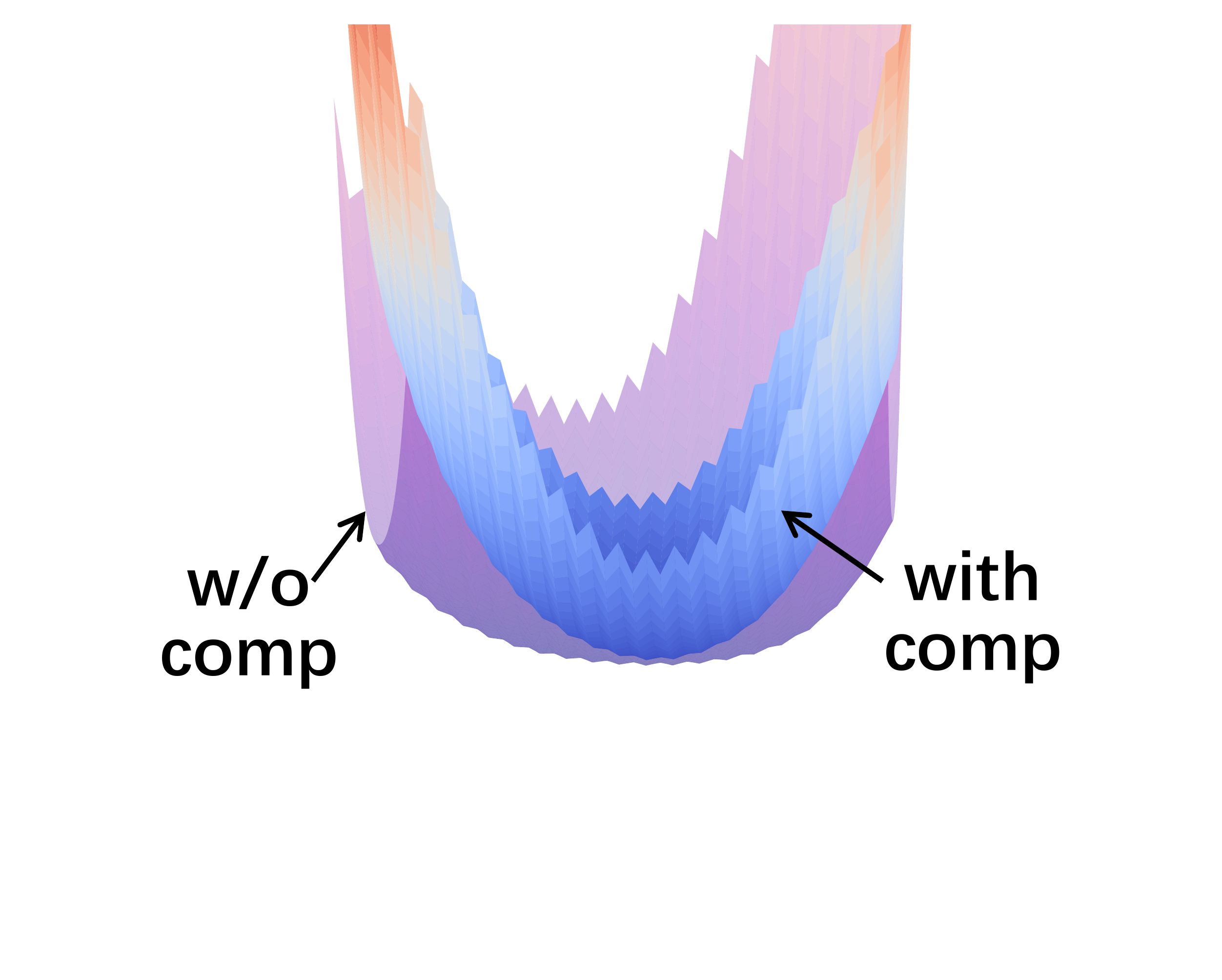}
        \vspace{-0.5in}
        \caption{8 bit Quantization,\\$\textit{Dir}(0.01)$}
        \label{fig:image4}
    \end{subfigure}
    
    \caption {Visualization of loss landscapes of FedAvg with and without model update compression, where the transparent purple loss surface arrowed by ``w/o comp" in each figure corresponds to FedAvg without compression. The experiments are conducted on the Fashion-MNIST dataset with stochastic quantization \cite{alistairh2017qsgd} and Top-$k$ sparsification \cite{alistairh2018convergence}. We allocate the training data to 100 clients following the uniform distribution to simulate IID data distribution, and following the Dirichlet distribution (\textit{Dir}) \cite{hsu2019measuring} to simulate data heterogeneity.}
    \vspace{-0.2in}
    \label{fig:fig_pre}
\end{figure}

\begin{table}[h]
\centering
\small
\caption{Comparison of Hessian top eigenvalues under different data distributions and compression settings. Higher eigenvalues indicate sharper loss landscapes.}
\label{tab:hessian_eigenvalues}
\begin{tabular}{llc}
\toprule
  & \textbf{Compression Setting} & \textbf{Top Eigenvalue} \\
\midrule
\multirow{4}{*}{IID} & w/o Compression & 7.56 \\
 & 8-bit Quantization & 8.36 \\
 & Top-$k$ (0.25 sparsity) & 8.50 \\
 & 4-bit Quantization & 8.82 \\
\midrule
\multirow{2}{*}{$Dir(0.01)$} & w/o Compression & 17.04 \\
 & 8-bit Quantization & 17.78 \\
\bottomrule
\end{tabular}
\end{table}

One potential solution to the problem is incorporating the recently emerging Sharpness-Aware Minimization (SAM) \cite{foret2020sharpness}, which is designed to improve the flatness of the loss landscape. 
\begin{figure}[t]
  \begin{center}
    \includegraphics[width=0.4\textwidth]{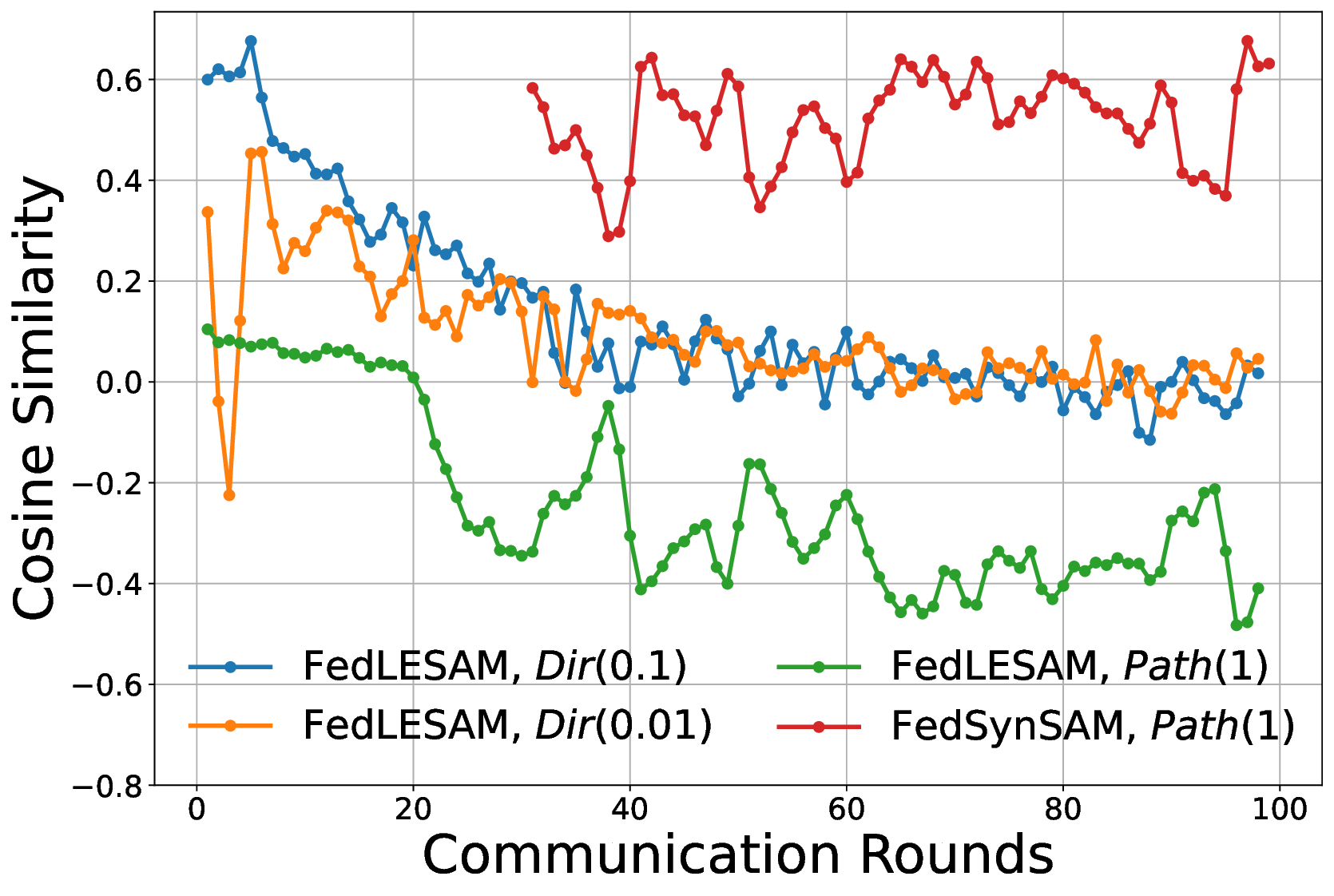}
  \end{center}
  \vspace{-0.1in}
  \caption{Cosine similarity between the true and estimated global perturbation in FedLESAM \cite{fan2024locally} and FedSynSAM. The experiments are conducted on the CIFAR-10 dataset with 4-bit stochastic quantization \cite{alistairh2017qsgd}. We simulate data heterogeneity with the Dirichlet distribution (\textit{Dir}) and Pathological distributions (\textit{Path}) \cite{hsu2019measuring}.}
  \label{cossim}
  \vspace{-0.2in}
\end{figure}
SAM admits a two-step optimization framework, which consists of a gradient ascent step that perturbs the model with the gradient, i.e., $\tilde{\bm{w}}=\bm{w}+\rho\frac{\nabla F(\bm{w})}{\|\nabla F(\bm{w})\|}$, where $\rho$ is a hyperparameter, followed by a gradient descent step given the gradient of the perturbed model $\nabla F(\tilde{\bm{w}})$. FedSAM \cite{qu2022generalized} adopts a direct application of SAM in FL by replacing SGD with SAM at clients, in which the local gradient on the local dataset $\nabla F_i(\bm{w})$ for each client $i$ serves as an estimate of the true global gradient $\nabla F(\bm{w})$ in the gradient ascent step. However, in the presence of data heterogeneity, the local gradient can significantly deviate from the true global gradient, thus resulting in difference between the local perturbation $\rho\frac{\nabla F_i(\bm{w})}{\|\nabla F_i(\bm{w})\|}$ and the expected true global perturbation $\rho\frac{\nabla F(\bm{w})}{\|\nabla F(\bm{w})\|}$, undermining the effectiveness of SAM. To cope with this problem, FedSMOO \cite{sun2023dynamic} introduces dynamic regularization at each client to correct the local perturbation, at the cost of doubled uplink communication overhead. FedLESAM \cite{fan2024locally} proposes to use the model update in the previous round to estimate the global perturbation. However, such a rough estimation is far from being accurate with gradient compression in non-independent and identically distributed (non-IID) data distribution settings.

In this work, we leverage the global model trajectory to estimate the true global perturbation with enhanced precision. Specifically, inspired by prior works on dataset distillation \cite{cazenavette2022dataset}, we utilize a synthetic dataset, derived from the global model trajectory, to encapsulate the knowledge of the global dataset by aligning the trajectories of the global model and those produced by the synthetic dataset. Figure \ref{cossim} shows the cosine similarity between the true and estimated global perturbation for both FedLESAM and the proposed method, which implies that the synthetic dataset yields a more accurate estimation. We summarize our contributions as follows.
\begin{itemize}
    \item By inspecting the impact of gradient compression on the sharpness of loss landscapes, we reveal that improvement in communication efficiency by compression may come at the cost of degradation in generalization capability. We then identify SAM as a promising candidate to alleviate the concern.
    
    \item Recognizing the discrepancy between the local and true global perturbation in FedSAM, especially when handling heterogeneous data, we propose FedSynSAM. More specifically, FedSynSAM utilizes a synthetic dataset, derived via trajectory matching without requiring extra information beyond the global model trajectory, to facilitate a more accurate estimation of the true global perturbation.

    \item The convergence of FedSynSAM is established, and the impact of the perturbation estimation precision is characterized. Extensive experiments are conducted to validate the effectiveness of the proposed algorithm.
\end{itemize}

\section{Related Work}

\noindent\textbf{Communication overhead.} There are two principal approaches that mitigate the communication overhead challenges in FL: local update and gradient compression. The local update scheme requires each client to perform multiple local updates before information exchange, reducing the frequency of communication with the central server \cite{mcmahan2017communication,Chen2020TNNLS}. Gradient compression mechanisms, on the other hand, aim to reduce the size of transmitted gradients.
Unbiased compressors \cite{alistairh2017qsgd,reisizadeh2020fedpaq} have garnered considerable attention since the convergence of the corresponding algorithms can be readily established. Biased compressors, such as Top-$k$ sparsification \cite{alistairh2018convergence}, deterministic rounding \cite{sapio2021scaling}, and signSGD \cite{bernstein2019signsgd}, produce biased estimates of the true gradients, and the corresponding algorithms require case-by-case careful convergence analyses. Nonetheless, to the best of our knowledge, the impact of gradient compression on the generalization performance remains less understood.

\noindent\textbf{Data heterogeneity.} Unlike traditional distributed training, federated learning confronts the challenge of heterogeneous data distributions, reflecting real-world scenarios where client data is highly skewed \cite{ouyang_learning,zhaotccn}. 
A branch of prior research addresses this challenge by regularizing local training to mitigate local model divergence \cite{li2020federated,karimireddy2020scaffold,acar2021federated,zhao2018federated,li2022federated}. 
Existing works have incorporated dataset distillation approaches into FL with heterogeneous data. In FedDM \cite{xiong2023feddm}, each client condenses its local dataset through gradient distribution matching \cite{zhao2023dataset}, then uploads this distilled dataset to the server for aggregation and updating the global model. DynaFed \cite{pi2023dynafed} utilizes trajectory matching \cite{cazenavette2022dataset} to generate distilled data based on the global model's trajectory at the server, subsequently updating the global model with this synthetic dataset. FedDGM \cite{jia2024unlocking} extends DynaFed by initializing synthetic datasets with a pretrained generative model, then condensing personalized synthetic datasets for each client based on their model updates. In contrast, we incorporate the distilled dataset into SAM to seek for a flatter loss landscape.

\noindent\textbf{SAM-based FL algorithms.} Recent investigations have approached the data heterogeneity problem from an alternative perspective, revealing that data heterogeneity induces sharper loss landscapes in FL \cite{qu2022generalized}. To address this issue, SAM \cite{foret2020sharpness} has been deployed to identify flat minima for improved generalization performance, leading to the development of FedSAM \cite{qu2022generalized,caldarola2022improving}.
Since FedSAM complements regularization-based methods, subsequent research has combined these approaches to enhance performance. MoFedSAM \cite{qu2022generalized} introduces momentum into local updates, while FedGAMMA \cite{dai2023fedgamma} integrates the control variate from SCAFFOLD \cite{karimireddy2020scaffold} into FedSAM. However, under non-IID settings, there exists a discrepancy between the local perturbation in FedSAM and the desired global perturbation. With such consideration,
FedSMOO \cite{sun2023dynamic} introduces a dynamic regularizer to correct the local perturbation through the Alternating Direction Method of Multipliers (ADMM), which introduces additional communication overhead. Recognizing the communication constraints, the state-of-the-art approach FedLESAM \cite{fan2024locally} attempts to estimate the global perturbation by leveraging the global model updates in the previous communication round. However, such a rough estimation is far from being accurate (c.f. Figure \ref{cossim}). Moreover, the aforementioned works do not take gradient compression into consideration.

\section{Preliminaries}

\subsection{Communication-Efficient Federated Learning}
\noindent We focus on a federated learning architecture where $N$ clients aim at training a global model under the coordination of a parameter server. We assume that each client $i \in [N]$ holds a local dataset containing $m_i$ training data samples, denoted by $\mathcal{D}_i=\{d_{i,j}\}_{j=1}^{m_i}$, where $d_{i,j}$ is the $j$-th data point. Note that $\mathcal{D}_i$ may differ across different clients due to data heterogeneity. 
Let $F_i(\bm{w})$ be the training loss function of the client $i$, i.e.,  $F_i(\bm{w})\triangleq F(\bm{w},\mathcal{D}_i)=\frac{1}{m_i}\sum_{j=1}^{m_i}l(\bm{w},d_{i,j})$, where $\bm{w}$ is the global model weight and $l(\cdot,\cdot)$ the per sample loss function. FL \cite{li2019convergence} aims to find the best global model $\bm{w}^*$ via solving the following empirical risk minimization (ERM) problem:
\begin{equation}\label{fl_problem}
    \bm{w}^*=\arg\underset{\bm{w}}{\min}~{F(\bm{w})},
\end{equation}
where $F(\bm{w})=\frac{1}{N}\sum_{i=1}^NF_i(\bm{w})$ is the loss function of the global model.

\textbf{FedAvg}\cite{mcmahan2017communication}. At round $t\in\{1,\cdots,T\}$, $S$ random clients $\S^t$ are sampled. Each sampled client receives the global model $\bm{w}^t$ from the server and performs $K$ local training iterations:
\begin{equation}\label{avg_update}
    \bm{w}^t_{i,k+1}=\bm{w}^t_{i,k}-\eta_l\bm{g}_{i,k}^t, k \in \{1,...,K\},
\end{equation}
where $\eta_l$ is the local learning rate, $\bm{w}_{i,k}^t$ is client $i$'s local model after $k$ local updates in round $t$ with $\bm{w}_{i,0}^t=\bm{w}^t$, and $\bm{g}_{i,k}^t= \nabla F(\bm{w}_{i,k}^t, \xi_{i,k}^t)$ is an unbiased estimate of the local gradient $\nabla F_i(\bm{w}_{i,k}^t)$ using random sample $\xi_{i,k}^t$ from local dataset $\mathcal{D}_i$.
After $K$ local iterations, client $i$ uploads the model update $\Delta^t_i=\bm{w}^t_{i,K}-\bm{w}^t$ to the server for aggregation, i.e., 
$ \bm{w}^{t+1}=\bm{w}^t+\frac{\eta_g}{S}\sum_{i\in \S^t}\Delta^t_i$,
where $\eta_g$ is the global learning rate.

\textbf{Gradient Compression.} To reduce communication costs, instead of transmitting $\Delta^t_i$, clients can share a compressed version $Q(\Delta^t_i)$ with the server, where $Q(\cdot)$ denotes a compressor, resulting in the global update 
$\bm{w}^{t+1}=\bm{w}^t+\frac{\eta_g}{S}\sum_{i\in \S^t}Q(\Delta^t_i)$.

In this work, we consider two widely adopted gradient compressors, namely stochastic quantization and Top-$k$ sparsification as follows.

\textbf{Stochastic quantization \cite{alistairh2017qsgd}.}
For any vector $\bm{v}$, the stochastic quantization operator $Q(\bm{v})$ is defined as
\begin{equation}
    Q(\bm{v})=||\bm{v}||_2\cdot \text{sign}(v_i)\xi_i(\bm{v},a),
\end{equation}
where $a = 2^b + 1$ denotes the number of quantization levels and $b$ represents the number of bits used for quantization. Each $\xi_i(\bm{v}, a)$ is an independent random variable determined as follows. Let $0\leq l<a$ be an integer such that $|v_i|/||\bm{v}||_2\in [l/a,(l+1)/a]$, then
\begin{equation}
    \xi_i(\bm{v},a)=
    \begin{cases}
        l/a \hfill~~~ &\text{with probability}~ 1-p(\frac{|v_i|}{||\bm{v}||_2},a),\\
        (l+1)/a  &\text{otherwise},
    \end{cases}
\end{equation}
where $p(j, a) = ja - l$ for $j \in [0, 1]$.
Stochastic quantization is unbiased, i.e., $\mathbb{E}[Q(\bm{v})] = \bm{v}$.

\textbf{Top-$k$ sparsification.}
Given a sparsity ratio $k\in[0,1]$, this compressor preserves the top $k$ fraction of elements in $\bm{v}$ with the largest magnitudes while setting the remaining entries to zero. This approach effectively reduces communication overhead by transmitting only the non-zero coordinates and their corresponding indices.

\subsection{Sharpness Aware Minimization (SAM)}
Several studies \cite{hochreiter1994simplifying,keskar2016large,jiang2019fantastic} have revealed that a flat surface of the loss function tends to exhibit superior generalization performance. SAM \cite{foret2020sharpness} seeks for a region that has both small loss and flat loss landscape, which can be formulated as a min-max optimization problem, i.e.,
\begin{equation}\label{sam}
    \underset{\bm{w}}{\min}\underset{||\bm{\epsilon}|| \leq \rho}{\max}F(\bm{w}+\bm{\epsilon}),
\end{equation}
where the inner optimization problem attempts to find a perturbation $\bm{\epsilon}$ in a $\mathcal{\ell}_2$ Euclidean ball with a pre-defined radius $\rho$ that maximizes the loss function $F(\bm{w}+\bm{\epsilon})$. To solve the inner problem, SAM approximates the optimal $\bm{\epsilon}^*$ using a first-order Taylor expansion as:
\begin{equation}\label{ep_of_sam}
\begin{aligned}
\bm{\epsilon}^* 
&= \arg\max_{\|\bm{\epsilon}\|\leq \rho} F(\bm{w}+\bm{\epsilon}) \\
&\approx \arg\max_{\|\bm{\epsilon}\|\leq \rho} 
   \big[ F(\bm{w}) + \bm{\epsilon}^T \nabla F(\bm{w}) \big] 
   \approx \rho \frac{\nabla F(\bm{w})}{\|\nabla F(\bm{w})\|}.
\end{aligned}
\end{equation}
Then, SAM obtains the gradient approximation as
$\nabla F^{\textit{SAM}}(\bm{w})=\nabla F(\tilde{\bm{w}})|_{\tilde{\bm{w}}=\bm{w}+\bm{\epsilon}^*}$.

FedSAM \cite{qu2022generalized} extends SAM to federated learning by applying the SAM locally at each client. Specifically, at each communication round $t$, client $i$ performs the following two-step update:
\begin{equation}\label{sam_update}
    \begin{cases}
        \tilde{\bm{w}}_{i,k}^t=\bm{w}_{i,k}^t+\rho {\frac{\nabla F_i(\bm{w}_{i,k}^t)}{||\nabla F_i(\bm{w}_{i,k}^t)||}}\\
        \bm{w}_{i,k}^{t+1}=\bm{w}_{i,k}^t-\eta_l \nabla F_i(\tilde{\bm{w}}_{i,k}^t).
    \end{cases}
\end{equation}
In practice, the stochastic gradient $\bm{g}_{i,k}^t= \nabla F(\bm{w}_{i,k}^t, \xi_{i,k}^t)$ and $\tilde{\bm{g}}_{i,k}^t= \nabla F(\tilde{\bm{w}}_{i,k}^t, \xi_{i,k}^t) $ serve as unbiased estimates of the gradients $\nabla F_i(\bm{w}_{i,k}^t)$ and $\nabla F_i(\tilde{\bm{w}}_{i,k}^t)$, where $\xi_{i,k}^t$ represents a random sample from the local dataset. 
Formally, for clarity of presentation, we define the local perturbation of client $i$ at round $t$ as $\bm{\epsilon}_{i}^t=\rho \frac{\nabla F_i(\bm{w}^t)}{||\nabla F_i(\bm{w}^t)||}$, and the global perturbation at round $t$ as $\bm{\epsilon}^t=\rho \frac{\nabla F(\bm{w}^t)}{||\nabla F(\bm{w}^t)||}$.

\section{Proposed Algorithm: FedSynSAM}\label{algorithm_section}

\subsection{Rethinking Federated Learning with Gradient Compression.}
\noindent When a compressor $Q(\cdot)$ is adopted, the global model update in the last round can be written as $\bm{w}^{T}=\bm{w}^{T-1}+\frac{\eta_g}{S}\sum_{i\in \S^t}\Delta^{T-1}_i+\eta_g\bm{\delta}$, where $\bm{\delta}=\frac{1}{S}\sum_{i\in \S^t} \bm{\delta}_i$ and $\bm{\delta}_i=Q(\Delta_i^{T-1})-\Delta_i^{T-1}$ represents the compression error for client $i$. In the IID data distribution scenario, as the global model converges, the local model updates $\Delta_i^{T-1}$ (and therefore the compression error $\bm{\delta}_{i}$) approach zero, rendering the impact of compression error negligible. However, in the presence of data heterogeneity, the local updates may not necessarily decay to 0. Consequently, the compression error $\bm{\delta}$ shifts the converged model to a random neighboring point near the optimal solution $\bm{w}^*$, where the randomness comes from the stochasticity in compressors, e.g., stochastic quantization \cite{alistairh2017qsgd}. With such consideration, a flat loss landscape would improve the robustness against compression errors, and SAM becomes a natural choice.
\subsection{Accurate Estimation of Global Perturbation via Synthetic Dataset}\label{mtt}
Note that, for SAM, the perturbation $\bm{\epsilon}^t$ depends on the global gradient $\nabla F(\bm{w}^t)=\frac{1}{N}\sum_{i=1}^N\nabla F(\bm{w}^t,\mathcal{D}_i)$, which is determined by the global dataset $\mathcal{D}_{global}=\{\mathcal{D}_i\}_{i=1}^N$. However, directly sharing the local dataset with the server violates the data minimization principle of FL. To this end, we propose constructing a synthetic dataset based on the global model trajectory, i.e., no extra information beyond the vanilla FL framework is required. This synthetic dataset can then be utilized to estimate the global perturbations.
Specifically, the corresponding optimization problem can be formulated as:
\begin{equation}\label{optim_gm}
    \underset{\mathcal{D}_{syn}}{\min}\mathbb{E}_t\|\nabla F(\bm{w}^t,\mathcal{D}_{syn})-\nabla F(\bm{w}^t)\|^2,
\end{equation}
where $\mathcal{D}_{syn}=\{X,Y\}$ denotes the synthetic dataset. Here, the label set $Y$ follows a uniform distribution over the class labels $\{0,\cdots,c-1\}$, with $c$ being the number of classes. The feature set $X$ is initialized either with random noise or with samples generated by a pretrained generative model, as suggested in \cite{cazenavette2023generalizing,jia2024unlocking}. Once $\mathcal{D}_{syn}$ is constructed, we can estimate the global perturbation with $\nabla F(\bm{w}^t,\mathcal{D}_{syn})$ at each client. 

A major difficulty in solving (\ref{optim_gm}) lies in the fact that the global gradients $\nabla F(w)$'s are not available beforehand. Inspired by \cite{cazenavette2022dataset,pi2023dynafed} and considering that global model updates effectively represent the evolving trajectory of the global model in FL, we relax the above synthetic dataset construction task to a trajectory matching problem. Specifically, considering a segment of the global model's trajectory from $\bm{w}^t$ to $\bm{w}^{t+s}$, the problem can be interpreted as follows: starting from $\bm{w}^t$, the model should reach a point close to $\bm{w}^{t+s}$ after being trained for $s$ iterations on $\mathcal{D}_{syn}$. The optimization problem can then be reformulated as: 
\begin{align}
    &\underset{{X,\alpha}}{\min}~\mathbb{E}_t[\|\hat{\bm{w}}^{t+s}-\bm{w}^{t+s}\|^2],\label{out_loop}\\
    s.t.&~~ \hat{\bm{w}}^{t+s}=\mathcal{A}(X,Y,\bm{w}^{t},\alpha,s),\label{in_loop}
\end{align}
where $\hat{\bm{w}}^{t+s}$ represents the model trained by a trainer $\mathcal{A}$ on $\mathcal{D}_{syn}$ from initial weight $\bm{w}^{t}$ with learning rate $\alpha$ for $s$ iterations. Specifically, each training iteration by the trainer $\mathcal{A}$ can be expressed as:
\begin{equation}
    \hat{\bm{w}}^{t+n+1}=\hat{\bm{w}}^{t+n}-\alpha \nabla F(\hat{\bm{w}}^{t+n},\{X,Y\}),~~ n\in \{0,1,\cdots,s-1\}
\end{equation}
where $\alpha$ is the learnable learning rate and $\hat{\bm{w}}^{t}=\bm{w}^{t}$. To minimize the mean square error (MSE) between $\hat{\bm{w}}^{t+s}$ and $\bm{w}^{t+s}$ in 
(\ref{out_loop}), we optimize $X$ and the learning rate $\alpha$ by gradient decent, which are given by:
\begin{align}
    X=X-\eta_x\nabla_X \mathcal{L}(\hat{\bm{w}}^{t+s},\bm{w}^{t+s}),\\
    \alpha=\alpha-\eta_\alpha\nabla_\alpha \mathcal{L}(\hat{\bm{w}}^{t+s},\bm{w}^{t+s}),
\end{align}
where $\mathcal{L}(\cdot)$ denotes the MSE loss, $\eta_x$ and $\eta_\alpha$ represent the step sizes of updating the feature set $X$ and the learning rate $\alpha$, respectively. When gradient compression is employed, the global model trajectory is inevitably perturbed by compression-induced errors in each communication round, which may degrade the quality of the synthetic dataset distilled from these trajectories. Our experimental results in Section \ref{experiments} suggest that the constructed synthetic data still provides satisfactory performance for commonly considered compressors.
In addition, it is important to note that the synthetic dataset construction process only relies on the global model updates, which are available at both the server and the clients. 
\setlength{\textfloatsep}{3pt}
\setlength{\intextsep}{3pt}
\setlength{\abovecaptionskip}{1pt}
\setlength{\belowcaptionskip}{1pt}
\begin{algorithm}[t]
\caption{FedSAM / FedLESAM / FedSynSAM}\label{ALG:three}
\small
\setlength{\abovedisplayskip}{2pt}
\setlength{\belowdisplayskip}{2pt}
\setlength{\abovedisplayshortskip}{1pt}
\setlength{\belowdisplayshortskip}{1pt}
\begin{algorithmic}[1]
\STATE \textbf{Input:} initial model $\bm{w}_0$, radius $\rho$, local/global LR $\eta_l$/$\eta_g$, compressor $Q$, synthetic data $\mathcal{D}_{syn}=\{X,Y\}$, cache list $W=\{\bm{w}^0\}$, local steps $K$, condensation iters $M$, condensation LR $\eta_{x},\eta_\alpha$, rounds $T$.

\FOR{$t=0$ to $T-1$}
    \STATE Sample a random set $\S^t$ with $S$ clients 
    \FOR{each client $i\in \S^t$ in parallel}
        \STATE Initialize $\bm{w}^t_{i,0} = \bm{w}^t$
        \FOR{$k=0$ to $K-1$}
            \STATE Sample $\xi_{i,k}^t$ from $\mathcal{D}_i$, $\zeta_{i,k}^t$ from $\mathcal{D}_{syn}$
\Ge
\State
\begin{minipage}{0.83\linewidth}
\colorbox{fedblue!15}{%
  \parbox{\dimexpr\linewidth-2\fboxsep\relax}{%
    \textcolor{fedblue}{\textbf{FedSAM:}} ~~~~~~~~$\bm{g}_{i,k}^t = \nabla F(\bm{w}_{i,k}^t, \xi_{i,k}^t)$%
  }%
}
\end{minipage}

\State
\begin{minipage}{0.83\linewidth}
\colorbox{fedgreen!18}{%
  \parbox{\dimexpr\linewidth-2\fboxsep\relax}{%
    \textcolor{fedgreen}{\textbf{FedLESAM:}} ~~~$\bm{g}_{i,k}^t = \bm{w}^{t-1} - \bm{w}^t.$%
  }%
}
\end{minipage}

\State
\begin{minipage}{0.83\linewidth}
\colorbox{fedorange!18}{%
  \parbox{\dimexpr\linewidth-2\fboxsep\relax}{%
    \textcolor{fedorange}{\textbf{FedSynSAM:}}\\[-2pt]
    \(
    \bm{g}_{i,k}^t =
    \begin{cases}
      \nabla F(\bm{w}_{i,k}^t, \xi_{i,k}^t), & t \le R,\\[4pt]
      \begin{aligned}
        & \beta \nabla F(\bm{w}_{i,k}^t, \xi_{i,k}^t)\\
        & ~+ (1-\beta)\nabla F(\bm{w}_{i,k}^t, \zeta_{i,k}^t)
      \end{aligned}
      ,& t > R.
    \end{cases}
    \)
  }%
}
\end{minipage}



\EndGe

            \STATE \textbf{Universal two-step update:}
            \[
            \begin{aligned}
            \tilde{\bm{w}}^{\,t} = &\bm{w}_{i,k}^t
               + \rho \frac{\bm{g}_{i,k}^t}{\|\bm{g}_{i,k}^t\|}, \\
            \bm{w}_{i,k+1}^t &= \bm{w}_{i,k}^t
               - \eta_l \,\tilde{\bm{g}}_{i,k}^t \, .
            \end{aligned}
            \]
        \ENDFOR
        \STATE $\Delta_i^t = \bm{w}_{i,K}^t - \bm{w}^t$, upload $Q(\Delta_i^t)$
    \ENDFOR
\Server
    \STATE $\bm{w}^{t+1} = \bm{w}^t + \frac{\eta_g}{S}\sum_{i\in\S^t} Q(\Delta_i^t)$
    \IF{$t\le R$}
        \STATE Append $\bm{w}^{t+1}$ to $W$
    \ENDIF
    \IF{$t=R$}
        \FOR{$m=0$ to $M$}
            \STATE Sample $\bm{w}^r$ and $\bm{w}^{r+s}$ from $W$
            \STATE $\hat{\bm{w}}^{r+s}=\mathcal{A}(X,Y,\bm{w}^{r},\alpha,s)$
            \STATE $\mathcal{L}=\|\hat{\bm{w}}^{r+s}-\bm{w}^{r+s}\|^2$
            \STATE Update $X = X - \eta_{x}\nabla_X \mathcal{L}$, $\alpha = \alpha - \eta_{\alpha}\nabla_{\alpha}\mathcal{L}$
        \EndFor
        \STATE Send $\mathcal{D}_{syn}$ to clients
    \EndIf
\EndServer
\ENDFOR
\end{algorithmic}
\end{algorithm}

\subsection{Overall Algorithm of FedSynSAM}
The details of the proposed mechanism are summarized in Algorithm \ref{ALG:three} and discussed in the following. During the initial $R$ aggregation rounds, the server stores the global model parameters to construct the global model trajectory $W=\{\bm{w}^t\}_{t=0}^R$. 
The server then optimizes the synthetic dataset as described in Section \ref{mtt} and distributes the derived $\mathcal{D}_{syn}$ to all clients.
Each client updates its local model through $K$ iterations using the two-step framework defined in line 12 in Algorithm \ref{ALG:three}, with modified gradient calculation as follows:
\begin{equation}
    \nabla \hat{F}_i(\bm{w}_{i,k}^t) = \beta\nabla F(\bm{w}_{i,k}^t , \mathcal{D}_i)+(1-\beta)\nabla F(\bm{w}_{i,k}^t , \mathcal{D}_{syn}),
\end{equation}
where $\beta$ serves as an interpolation coefficient, balancing the contribution from the synthetic dataset and local dataset. 
Then the client estimate the global perturbation as $\hat{\bm{\epsilon}}_{i,k}^t=\rho \frac{\nabla \hat{F}_i(\bm{w}_{i,k}^t)}{\|\nabla \hat{F}_i(\bm{w}_{i,k}^t)\|}$. In practice, the stochastic gradient $\bm{g}_{i,k}^t =\beta \nabla F(\bm{w}_{i,k}^t, \xi_{i,k}^t) + (1-\beta)\nabla F(\bm{w}_{i,k}^t, \zeta_{i,k}^t)$ serves as an unbiased estimate of $\nabla \hat{F}_i(\bm{w}_{i,k}^t)$ as shown in line 11 in Algorithm \ref{ALG:three}, where $\xi_{i,k}^t$ and $\zeta_{i,k}^t$ are random samples from $\mathcal{D}_i$ and $\mathcal{D}_{syn}$, respectively.
Notably, in the initial $R$ rounds when $\mathcal{D}_{syn}$ is unavailable, gradient estimation follows $\nabla F(\bm{w}_{i,k}^t, \xi_{i,k}^t)$, consistent with FedSAM.

\section{Convergence Analysis}\label{analysis_section}
\noindent In this section, we provide the convergence guarantee of the FedSynSAM algorithm with the unbiased gradient compressor under general non-convex FL settings. While convergence with biased compressors is challenging and require case-by-case investigation \cite{Beznosikov2023BiasedCompression,Richtarik2021EF21}, we follow prior works \cite{Li2023FedEF,reisizadeh2020fedpaq,Li2020ADIANA,Condat2025LoCoDL} and focus on unbiased compressors for tractability, and our experiments further confirm that FedSynSAM is effective under both biased and unbiased settings. Before that, we introduce the following assumptions.

\begin{Assumption}\label{ass:smooth}
		(Smoothness). $F_i$ is $L$-smooth for all $i \in [N]$, i.e., 
		\begin{equation}
			\|\nabla F_i (\bm{w}) - \nabla F_i (\bm{v}) \| \leq L \|\bm{w}-\bm{v} \|,
		\end{equation}
		for all $\bm{w}, \bm{v}$ in its domain and $i \in [N]$.
	\end{Assumption}
	
	\begin{Assumption}\label{ass:sigmag}
		(Bounded variance of global gradient without perturbation). The global variability of the local gradient of the loss function without perturbation $\bm{\epsilon}_i$ is bounded by $\sigma_g^2$, i.e., 
		\begin{equation}
			\|\nabla F_i (\bm{w}^t ) - \nabla F(\bm{w}^t )\|^2 \leq \sigma_g^2, 
		\end{equation}
		for all $i \in [N]$ and $t$.
	\end{Assumption}

	\begin{Assumption}\label{ass:sigmal}
		(Bounded variance of stochastic gradient). The stochastic gradient $\nabla F (\bm{w}, \xi_i )$, computed by the $i$-th client of model parameter $\bm{w}$ using mini-batch $\xi_i$ is an unbiased estimator $\nabla F_i (\bm{w})$ with variance bounded by $\sigma^2$, i.e., \begin{equation}
			\mathbb{E}_{\xi_i}\left\|\frac{\nabla F (\bm{w},\xi_i )}{\|\nabla F (\bm{w},\xi_i )\|} - \frac{\nabla F_i (\bm{w})}{\|\nabla F_i (\bm{w})\|}\right\|^2 \leq \sigma_l^2,
		\end{equation}
		$\forall i \in [N]$, where the expectation is over all local datasets.
	\end{Assumption}

\begin{Assumption}\label{ass:quant_var}
    The random quantizer $Q(\cdot)$ is unbiased and its variance grows with the squared of $l_2$-norm of its argument, i.e.,
    \begin{equation}
        \mathbb{E}[Q(\bm{x})|\bm{x}]=\bm{x}, ~~~ \mathbb{E}[\Vert Q(\bm{x})-\bm{x} \Vert_2^2|\bm{x}]\leq q\Vert \bm{x} \Vert_2^2,
    \end{equation}
    for some positive real constant $q$ and any $\bm{x}\in \mathbb{R}^p$.
\end{Assumption}

Assumption \ref{ass:smooth} and \ref{ass:sigmag} are standard in general non-convex FL studies \cite{karimireddy2020scaffold,yang2021achieving,Xu2022TernaryCompression}. Assumption \ref{ass:sigmal} bounds the variance of the stochastic gradient as in \cite{qu2022generalized,fan2024locally}. Assumption \ref{ass:quant_var} ensures that the output of the compressor is an unbiased estimator of the input with a variance that is proportional to $\ell_2$-norm of the input \cite{reisizadeh2020fedpaq}.



In the following, we establish the convergence of FedSynSAM for both full client participation and partial client participation. The detailed proofs are provided in the Appendix. First of all, we introduce a lemma that bounds the deviation between the local and global gradient at the perturbed model impacted by the precision of the estimated global perturbation in FedSynSAM.

\begin{Lemma}\label{lemma:sigmag}
    (Bounded deviation of $\|\nabla F_i (\bm{w} + \hat{\bm{\epsilon}}_i ) - \nabla F(\bm{w} + \bm{\epsilon})\|^2$.) The deviation of local and global gradients with perturbation in FedSynSAM can be bounded as follows:
    \begin{equation}
        \|\nabla F_i (\bm{w} + \hat{\bm{\epsilon}}_i ) - \nabla F(\bm{w} + \bm{\epsilon})\|^2 \leq \gamma ,
    \end{equation}
    where $\hat{\bm{\epsilon}}_i=\rho \frac{\nabla \hat{F}_i(\bm{w})}{\|\nabla \hat{F}_i(\bm{w})\|}, \bm{\epsilon}=\rho \frac{\nabla F(\bm{w})}{\|\nabla F(\bm{w})\|}$, $\nabla \hat{F}_i(\bm{w})=\beta\nabla F_i(\bm{w})+(1-\beta)\nabla F(\bm{w}, \mathcal{D}_{syn})$, $\gamma=2\sigma_g^2+4L^2\rho^2(1-\cos\theta)$, and $\theta=\arccos{\big(\frac{\nabla \hat{F}_i(\bm{w})\cdot \nabla F(\bm{w})}{\|\nabla \hat{F}_i(\bm{w})\|\| \nabla F(\bm{w})\|}\big)}$ is the angle between the estimated global perturbation and the true global perturbation.
\end{Lemma}

\begin{Remark}\label{remark1}
    From Lemma \ref{lemma:sigmag} we observe that 
    as the quality of the synthetic dataset improves, $\nabla F(\bm{w}, \mathcal{D}_{syn})$ approaches $\nabla F(\bm{w})$, reducing the angle $\theta$ and thus $\gamma$.
    More specifically, the optimization error of estimating the synthetic dataset can be quantified as $ \kappa = \mathbb{E}_t\Big\|\nabla F(\bm{w}^t,\mathcal{D}_{syn})-\nabla F(\bm{w}^t)\Big\|^2$
    according to (\ref{optim_gm}). By the geometric inequality $1-\cos\theta \;\le\;\frac{\|u-v\|_2^2}{2\|u\|_2\|v\|_2}$,
    with $u=\nabla \hat{F}_i(\bm{w})$ and $v=\nabla F(\bm{w})$, we obtain $1-\cos\theta = \mathcal{O}(\kappa)$.
   Therefore, reducing $\kappa$ directly decreases $(1-\cos\theta)$ and tightens the bound $\gamma$.
\end{Remark}

\begin{theorem}\label{full_ca}
		Let $\eta_l\leq\frac{1}{10KL}$, $\eta_g\eta_l \leq \min(\frac{1}{KL},\frac{N}{150Kq})$. Under Assumption~\ref{ass:smooth}-\ref{ass:quant_var} and with full client participation, which means $S=N$, the sequence of outputs $\{\bm{w}^t \}$ generated by FedSynSAM satisfies
    \begin{equation}
    \begin{aligned}
    \frac{1}{T}\sum_{t=1}^{T}\E &\big[\|\nabla F(\bm{w}^{t+1})\|^2\big]
    = \mathcal{O}\Bigg(
         \frac{F^*L}{\sqrt{TKN}}
       + \frac{\gamma}{T}  \\[4pt]
    &+ \frac{L^2 \sigma_l^2}{T^{3/2}\sqrt{KN}}
       + \frac{q\sqrt{K}\,\gamma}{\sqrt{TN}L} 
       + \frac{qL\sigma_l^2}{T^{3/2}\sqrt{KN}}
       \Bigg) .
    \end{aligned}
    \end{equation}
    where $F^*= F(\tw^0)-F(\tw^*)$ and $F(\tw^*)=\min_{\tw}F(\tw)$.
	\end{theorem}

\begin{theorem}\label{partial_ca}
		Let $\eta_l$ and $\eta_g$ be chosen as such that $\eta_l \leq \frac{1}{10KL}$, $\eta_g\eta_l\leq\frac{1}{KL}$ and the condition $\eta_l \eta_g \leq \frac{S}{12KL + 150Kq}$ holds. Under Assumption~\ref{ass:smooth}-\ref{ass:quant_var}, with partial client participation and each client sampled randomly and independently without replacement, the sequence of outputs $\{\bm{w}^t \}$ generated by FedSynSAM satisfies
    \begin{equation}
    \begin{aligned}
    \frac{1}{T}\sum_{t=1}^{T}\E\!&\left[\|\nabla F(\bm{w}^{t+1})\|^2\right]
    = \mathcal{O}\Bigg(
          \frac{F^*L}{\sqrt{TKS}}
        + \frac{\sqrt{K}\,\gamma}{\sqrt{TS}} \\[4pt]
    &
        + \frac{L^2 \sigma_l^2}{T^{3/2}K}
        + \frac{q\sqrt{K}\,\gamma}{\sqrt{TS}L}
        + \frac{qL\sigma_l^2}{T^{3/2}\sqrt{KS}}
        \Bigg).
    \end{aligned}
    \end{equation}
	\end{theorem}

\begin{Remark}
    For the full and partial client participation strategies of the FedSynSAM algorithm, the dominant terms of the convergence rate are both $\mathcal{O}(\frac{1}{\sqrt{T}})$ by properly choosing the learning rate $\eta_l$, which match the best convergence rate in existing general non-convex FL studies \cite{karimireddy2020scaffold,yang2021achieving,reisizadeh2020fedpaq}. 
\end{Remark}

\begin{Remark}
    Theorem \ref{full_ca} and \ref{partial_ca} recover the convergence rate of the FedSAM algorithm with full and partial client participation for non-convex objective functions as a special case of $q=0$, where $\gamma=\mathcal{O}(\sigma_g^2+ L^2\rho^2)$ according to Remark \ref{remark1}. The additional$ \frac{q\sqrt{K}\gamma}{\sqrt{TS}L}$ and $\frac{qL\sigma_l^2}{T^{3/2}\sqrt{KS}}$ terms come from the gradient compressor adopted.
\end{Remark}

\begin{Remark}
    According to Lemma \ref{lemma:sigmag}, a smaller deviation bound $\gamma$ corresponds to a smaller upper bound in both Theorem \ref{full_ca} and \ref{partial_ca}, which demonstrates the improvement of the proposed FedSynSAM due to the more precise estimation of the global perturbation. 
\end{Remark}

\begin{table*}[!t]
  \caption{Test accuracy of various methods under full and partial client participation. The bold results indicate the best accuracy, and the underlined results indicate the second-best accuracy.}
  \label{tab:comparison}
  \centering
  \renewcommand{\arraystretch}{0.86}  
\setlength{\abovecaptionskip}{2pt}  
\setlength{\belowcaptionskip}{2pt}
  \small
  \begin{tabular}{l@{\hspace{1.5pt}}c@{\hspace{3pt}}cccccccc}
    \toprule
    & & \multicolumn{4}{c}{Fashion-MNIST} & \multicolumn{4}{c}{CIFAR-10} \\
    \cmidrule(lr){3-6} \cmidrule(lr){7-10}
    & & \multicolumn{2}{c}{Stochastic Quantization} & \multicolumn{2}{c}{Top-$k$} & \multicolumn{2}{c}{Stochastic Quantization} & \multicolumn{2}{c}{Top-$k$} \\
    \cmidrule(lr){3-4} \cmidrule(lr){5-6} \cmidrule(lr){7-8} \cmidrule(lr){9-10}
    Method & \makecell{Comm. \\ Overhead} & 4bit & 8bit & 0.1 & 0.25 & 4bit & 8bit & 0.1 & 0.25 \\
    \midrule
    \multicolumn{10}{c}{\textbf{$\textit{Path}(1)$, 10 clients, full participation}} \\
    \midrule
    {\footnotesize FedAvg} & 1$\times$ & 81.31{\tiny  $\pm $ 0.14} & 81.19{\tiny  $\pm $ 0.25} & 72.84{\tiny  $\pm $ 0.84} & 77.90{\tiny  $\pm $ 0.34} & 68.55{\tiny  $\pm $ 0.38} & 68.38{\tiny  $\pm $ 0.29} & 57.50{\tiny  $\pm $ 0.74} & 62.72{\tiny  $\pm $ 0.71} \\
    {\footnotesize DynaFed} & - & 74.42{\tiny  $\pm $ 0.38} & 74.38{\tiny  $\pm $ 0.94} & 71.35{\tiny  $\pm $ 0.81} & 73.15{\tiny  $\pm $ 0.78} & 52.77{\tiny  $\pm $ 0.37} & 53.80{\tiny  $\pm $ 0.18} & 46.51{\tiny  $\pm $ 0.28} & 49.88{\tiny  $\pm $ 0.62} \\
    {\footnotesize FedSAM} & 1$\times$ & \underline{82.93}{\tiny  $\pm $ 0.18} & \underline{82.88}{\tiny  $\pm $ 0.19} & 75.85{\tiny  $\pm $ 0.49} & 80.40{\tiny  $\pm $ 0.17} & \underline{70.40}{\tiny  $\pm $ 0.05} & \underline{70.07}{\tiny  $\pm $ 0.64} & 57.78{\tiny  $\pm $ 0.55} & \underline{63.69}{\tiny  $\pm $ 0.24} \\
    {\footnotesize FedLESAM} & 1$\times$ & 82.47{\tiny  $\pm $ 0.13} & 82.37{\tiny  $\pm $ 0.09} & 74.81{\tiny  $\pm $ 0.72} & 79.01{\tiny  $\pm $ 0.16} & 69.69{\tiny  $\pm $ 0.40} & 69.70{\tiny  $\pm $ 0.31} & \underline{57.97}{\tiny  $\pm $ 0.35} & 63.49{\tiny  $\pm $ 0.36} \\
    {\footnotesize FedSMOO} & 2$\times$ & 78.99{\tiny  $\pm $ 0.45} & 78.57{\tiny  $\pm $ 0.53} & \underline{82.01}{\tiny  $\pm $ 0.44} & 80.48{\tiny  $\pm $ 0.47} & 47.13{\tiny  $\pm $ 0.52} & 48.85{\tiny  $\pm $ 0.65} & 42.89{\tiny  $\pm $ 0.85} & 42.66{\tiny  $\pm $ 0.52} \\
    {\footnotesize FedGAMMA} & 2$\times$ & 78.42{\tiny  $\pm $ 0.29} & 80.78{\tiny  $\pm $ 0.21} & \textbf{82.06}{\tiny  $\pm $ 0.38} & \underline{81.98}{\tiny  $\pm $ 0.11} & 65.18{\tiny  $\pm $ 0.37} & 65.96{\tiny  $\pm $ 0.47} & 49.58{\tiny  $\pm $ 0.27} & 59.44{\tiny  $\pm $ 0.86} \\
    {\footnotesize FedLESAM-D} & 2$\times$ & 68.22{\tiny  $\pm $ 1.13} & 63.40{\tiny  $\pm $ 3.44} & 65.27{\tiny  $\pm $ 1.25} & 66.05{\tiny  $\pm $ 1.41} & 26.77{\tiny  $\pm $ 0.83} & 28.30{\tiny  $\pm $ 0.79} & 27.70{\tiny  $\pm $ 1.30} & 27.42{\tiny  $\pm $ 1.49} \\
    {\footnotesize FedLESAM-S} & 2$\times$ & 79.71{\tiny  $\pm $ 0.35} & 77.61{\tiny  $\pm $ 0.70} & 75.44{\tiny  $\pm $ 0.66} & 77.43{\tiny  $\pm $ 0.13} & 63.31{\tiny  $\pm $ 0.60} & 62.37{\tiny  $\pm $ 0.38} & 57.02{\tiny  $\pm $ 0.51} & 61.03{\tiny  $\pm $ 0.34} \\
    {\footnotesize \textbf{FedSynSAM}} & 1$\times$ & \textbf{84.04}{\tiny  $\pm $ 0.19} & \textbf{83.93}{\tiny  $\pm $ 0.18} & 79.64{\tiny  $\pm $ 0.06} & \textbf{82.91}{\tiny  $\pm $ 0.21} & \textbf{71.56}{\tiny  $\pm $ 0.19} & \textbf{71.70}{\tiny  $\pm $ 0.32} & \textbf{63.60}{\tiny  $\pm $ 0.23} & \textbf{68.05}{\tiny  $\pm $ 0.18} \\

    \midrule
    \multicolumn{10}{c}{\textbf{$\textit{Dir}(0.01)$, 10 clients, full participation }} \\
    \midrule
    {\footnotesize FedAvg} & 1$\times$ & 83.20{\tiny  $\pm $ 0.46} & 83.40{\tiny  $\pm $ 0.17} & 81.00{\tiny  $\pm $ 0.21} & 82.33{\tiny  $\pm $ 0.23} & 75.00{\tiny  $\pm $ 0.18} & 74.91{\tiny  $\pm $ 0.03} & 66.79{\tiny  $\pm $ 0.42} & 71.51{\tiny  $\pm $ 0.16} \\
    {\footnotesize DynaFed} & - & 74.02{\tiny  $\pm $ 0.39} & 74.65{\tiny  $\pm $ 0.55} & 73.98{\tiny  $\pm $ 0.47} & 74.65{\tiny  $\pm $ 0.66} & 47.34{\tiny  $\pm $ 0.78} & 47.45{\tiny  $\pm $ 0.68} & 46.57{\tiny  $\pm $ 0.45} & 46.56{\tiny  $\pm $ 0.84} \\
    {\footnotesize FedSAM} & 1$\times$ & \underline{83.65}{\tiny  $\pm $ 0.14} & \underline{83.63}{\tiny  $\pm $ 0.11} & 81.28{\tiny  $\pm $ 0.13} & 82.61{\tiny  $\pm $ 0.14} & \underline{75.24}{\tiny  $\pm $ 0.39} & \underline{75.32}{\tiny  $\pm $ 0.24} & \underline{67.31}{\tiny  $\pm $ 0.40} & \underline{71.88}{\tiny  $\pm $ 0.27} \\
    {\footnotesize FedLESAM} & 1$\times$ & \underline{83.65}{\tiny  $\pm $ 0.13} & 83.56{\tiny $\pm $ 0.12} & 81.24{\tiny  $\pm $ 0.15} & 82.52{\tiny  $\pm $ 0.30} & 74.95{\tiny  $\pm $ 0.25} & 75.24{\tiny  $\pm $ 0.05} & 67.03{\tiny  $\pm $ 0.19} & 71.69{\tiny  $\pm $ 0.47} \\
    {\footnotesize FedSMOO} & 2$\times$ & 83.03{\tiny  $\pm $ 0.17} & 83.32{\tiny  $\pm $ 0.72} & \underline{81.84}{\tiny  $\pm $ 0.51} & \textbf{83.41}{\tiny  $\pm $ 0.18} & 70.09{\tiny  $\pm $ 0.54} & 74.59{\tiny  $\pm $ 1.02} & 64.84{\tiny  $\pm $ 0.37} & 70.95{\tiny  $\pm $ 0.95} \\
    {\footnotesize FedGAMMA} & 2$\times$ & 83.34{\tiny  $\pm $ 0.15} & 83.07{\tiny  $\pm $ 0.34} & 79.85{\tiny  $\pm $ 0.18} & 81.78{\tiny  $\pm $ 0.71} & 73.00{\tiny  $\pm $ 0.48} & 73.59{\tiny  $\pm $ 0.43} & 62.97{\tiny  $\pm $ 0.20} & 70.54{\tiny  $\pm $ 0.19} \\
    {\footnotesize FedLESAM-D} & 2$\times$ & 79.63{\tiny  $\pm $ 0.38} & 78.15{\tiny  $\pm $ 0.47} & 79.46{\tiny  $\pm $ 1.35} & 79.22{\tiny  $\pm $ 0.99} & 65.48{\tiny  $\pm $ 0.34} & 62.74{\tiny  $\pm $ 0.15} & 63.13{\tiny  $\pm $ 0.45} & 66.24{\tiny  $\pm $ 0.79} \\
    {\footnotesize FedLESAM-S} & 2$\times$ & 81.87{\tiny  $\pm $ 1.01} & 82.63{\tiny  $\pm $ 0.39} & 80.96{\tiny  $\pm $ 0.12} & 81.74{\tiny  $\pm $ 0.37} & 75.12{\tiny  $\pm $ 0.13} & 75.00{\tiny  $\pm $ 0.25} & 67.62{\tiny  $\pm $ 0.59} & 71.75{\tiny  $\pm $ 0.32} \\
    {\footnotesize \textbf{FedSynSAM}} & 1$\times$ & \textbf{84.17}{\tiny  $\pm $ 0.10} & \textbf{84.21}{\tiny  $\pm $ 0.08} & \textbf{81.86}{\tiny  $\pm $ 0.20} & \underline{83.38}{\tiny  $\pm $ 0.02} & \textbf{75.49}{\tiny  $\pm $ 0.10} & \textbf{75.51}{\tiny  $\pm $ 0.13} & \textbf{68.79}{\tiny  $\pm $ 0.08} & \textbf{72.66}{\tiny  $\pm $ 0.23} \\
        \midrule
    \multicolumn{10}{c}{\textbf{$\textit{Dir}(0.01)$, 50 clients, 20\% participation }} \\
    \midrule
    {\footnotesize FedAvg} & 1$\times$ & 80.25{\tiny  $\pm $ 0.09} & 80.24{\tiny  $\pm $ 0.07} & 77.54{\tiny  $\pm $ 0.40} & 79.75{\tiny  $\pm $ 0.19} & 71.32{\tiny  $\pm $ 0.05} & 71.47{\tiny  $\pm $ 0.03} &  67.11{\tiny  $\pm $ 0.06 } &  69.67{\tiny  $\pm $  0.13} \\
    {\footnotesize DynaFed} & - & 73.38{\tiny  $\pm $ 1.07} & 72.18{\tiny  $\pm $  0.78} & 67.84{\tiny  $\pm $ 0.12} & 70.94{\tiny  $\pm $ 0.49} & 52.78{\tiny  $\pm $ 0.47} & 52.91{\tiny  $\pm $ 0.22} & 51.39{\tiny  $\pm $ 1.52} & 51.53{\tiny  $\pm $ 0.14} \\
    {\footnotesize FedSAM} & 1$\times$ & 80.30{\tiny  $\pm $ 0.13} & 80.31{\tiny  $\pm $ 0.16} & 77.60{\tiny  $\pm $ 0.39} & 79.83{\tiny  $\pm $ 0.21} & 71.61{\tiny  $\pm $ 0.16} & 71.77{\tiny  $\pm $ 0.01} & \underline{67.21}{\tiny  $\pm $ 0.16} & 69.57{\tiny  $\pm $ 0.27} \\
    {\footnotesize FedLESAM} & 1$\times$ & 80.33{\tiny  $\pm $ 0.17} & 80.37{\tiny $\pm $ 0.13} & 77.64{\tiny  $\pm $ 0.33} & 79.79{\tiny  $\pm $ 0.20} & \underline{71.72}{\tiny  $\pm $ 0.21} & 71.67{\tiny  $\pm $ 0.12} &  67.07{\tiny  $\pm $ 0.45} & 69.71{\tiny  $\pm $ 0.06} \\
    {\footnotesize FedSMOO} & 2$\times$ & \textbf{82.95}{\tiny  $\pm $ 0.17} & \textbf{83.00}{\tiny  $\pm $ 0.23} & \textbf{81.86}{\tiny  $\pm $ 0.99} & \textbf{83.66}{\tiny  $\pm $0.04 } & 62.87{\tiny  $\pm $ 0.75} & 64.02{\tiny  $\pm $ 1.23} & 60.65{\tiny  $\pm $ 0.38} & 63.45{\tiny  $\pm $ 0.35} \\
    {\footnotesize FedGAMMA} & 2$\times$ &79.96{\tiny  $\pm $ 0.33} & 80.34{\tiny  $\pm $0.24 } & 76.19{\tiny  $\pm $ 0.78} & 78.81{\tiny  $\pm $ 0.40} & 67.27{\tiny  $\pm $ 0.56} & 68.00{\tiny  $\pm $ 0.51} & 63.06{\tiny  $\pm $ 0.36} & 66.42{\tiny  $\pm $ 0.56} \\
    {\footnotesize FedLESAM-D} & 2$\times$ & 80.44{\tiny  $\pm $ 0.21} & 79.78{\tiny  $\pm $ 0.46} & \underline{80.83}{\tiny  $\pm $0.28 } & 80.32{\tiny  $\pm $ 1.05} & 62.66{\tiny  $\pm $ 0.43} & 64.31{\tiny  $\pm $ 0.26} & 62.04{\tiny  $\pm $ 0.59} & 63.91{\tiny  $\pm $ 0.34} \\
    {\footnotesize FedLESAM-S} & 2$\times$ & 80.17{\tiny  $\pm $ 0.13} & \underline{81.09}{\tiny  $\pm $ 0.46} & 79.79{\tiny  $\pm $ 0.58} & 80.08{\tiny  $\pm $ 0.09} & 71.00{\tiny  $\pm $ 0.59} & \underline{71.95}{\tiny  $\pm $ 0.42} & 67.19{\tiny  $\pm $ 0.62} & \underline{69.78}{\tiny  $\pm $ 0.31} \\
    {\footnotesize \textbf{FedSynSAM}} & 1$\times$ & \underline{80.90}{\tiny  $\pm $ 0.04} & 80.89{\tiny  $\pm $ 0.11} & 78.33{\tiny  $\pm $ 0.30} & \underline{80.52}{\tiny  $\pm $ 0.21} & \textbf{72.74}{\tiny  $\pm $ 0.02} & \textbf{72.96}{\tiny  $\pm $ 0.12} & \textbf{70.53}{\tiny  $\pm $ 0.18} & \textbf{72.57}{\tiny  $\pm $ 0.36} \\
    \bottomrule
  \end{tabular}
  \vspace{-0.1in}
\end{table*}

\section{Experiments}\label{experiments}

\subsection{Experimental Setups}

\noindent\textbf{Datasets and models:} We use three image datasets: Fashion-MNIST \cite{xiao2017fashion}, CIFAR-10 \cite{krizhevsky2009learning} and CINIC-10 \cite{darlow2018cinic}. For data heterogeneity simulation, we follow \cite{qu2022generalized,fan2024locally} and use Dirichlet and Pathological splits \cite{hsu2019measuring} to simulate non-IID data distribution. For the Fashion-MNIST dataset, we adopt a two-layer MLP. For the CIFAR-10 and CINIC-10 datasets, we follow \cite{pi2023dynafed,yueqi2024fedredefense} and adopt ConvNet.

\textbf{Baselines:} We compare the proposed FedSynSAM with FedAvg \cite{mcmahan2017communication}, DynaFed \cite{pi2023dynafed}, FedSAM \cite{qu2022generalized}, FedGAMMA \cite{dai2023fedgamma}, FedSMOO \cite{sun2023dynamic} and the state-of-the-art FedLESAM/-S/-D \cite{fan2024locally}. It is worth noting that extra uplink communication overhead of transmitting the control variate or the dynamic correction is required in FedGAMMA, FedSMOO, FedLESAM-S and FedLESAM-D, and we assume that they are transmitted with the same compressor used for gradient compression in our experiments. For the compressor, we employ the widely adopted stochastic quantization method \cite{alistairh2017qsgd} with 4-bit and 8-bit quantization, as well as Top-$k$ sparsification \cite{alistairh2018convergence} with sparsity ratios of 0.1 and 0.25, respectively.

\textbf{Hyperparameter settings:} 
We fix the global learning rate to $1$ and tune the local learning rate from $\{0.01,0.05,0.1,0.5\}$. 
For the SAM optimizer, the perturbation radius $\rho$ is tuned within $\{0.001,0.01,0.05,0.1,0.5\}$. 
For the proposed FedSynSAM, we set the number of initial rounds to $R=30$ and select $\beta=0.9$, as it works well for all the examined scenarios. 
The synthetic dataset's feature set $X$ is initialized with Gaussian noise for Fashion-MNIST and with StyleGAN-generated samples~\cite{karras2019stylegan} for CIFAR-10 and CINIC-10, following recent approaches~\cite{cazenavette2023generalizing,jia2024unlocking}. 
We allocate $20$ images per class (IPC) for Fashion-MNIST and CIFAR-10, and $40$ IPC for CINIC-10, considering its larger diversity. 
For the optimization iteration of the synthetic dataset, we set $M$ to 200. 
The inner optimization steps are set to $s=5$ for CIFAR-10 and CINIC-10 with learning rates $\eta_\alpha=10^{-5}$ and $\eta_x=1000$ using SGD~\cite{cazenavette2022dataset}, while for Fashion-MNIST we use $s=3$, $\eta_\alpha=10^{-5}$, $\eta_x=0.05$, and Adam optimizer to construct the synthetic dataset. 
For client participation, we evaluate two scenarios: (i) full participation with all $10$ clients, and (ii) partial participation where $20\%$ of $50$ clients are randomly sampled each round. 
All algorithms adopt $10$ local iterations per communication round with a batch size of $128$. 
We run $1000$, $800$, and $300$ communication rounds for CIFAR-10, CINIC-10, and Fashion-MNIST, respectively. 
Each experiment is repeated three times, and the average test accuracy is reported.

\subsection{Main Results}
\textbf{Test accuracy:} The detailed experimental results on the Fashion-MNIST and the CIFAR-10 datasets are shown in Table \ref{tab:comparison}, and the test accuracy curves on CIFAR-10 over training rounds can be found in Figure \ref{fig:comparison}. As shown in Table \ref{tab:comparison}, when combined with gradient compression (i.e., stochastic quantization \cite{alistairh2017qsgd} and Top-$k$ sparsification \cite{alistairh2018convergence}), integrating SAM into local training improves the performance. In addition, FedLESAM does not necessarily outperform FedSAM due to inaccurate estimation of the global perturbation introduced by compression. The proposed FedSynSAM outperforms both FedSAM and FedLESAM in all the examined scenarios. While FedSMOO attains the highest test accuracy on Fashion-MNIST with partial participation at the cost of $2\times$ communication overhead, FedSynSAM demonstrates better performance on the more complex CIFAR-10 dataset. It is worth mentioning that other dynamic correction-based methods (i.e., FedLESAM-D) and control variate-based approaches (i.e., FedGAMMA, FedLESAM-S) do not deliver satisfactory performance, despite requiring additional communication costs to transmit dynamic corrections or control variates in each round. This may stem from the error in dynamic corrections or control variates due to compression. What's more, the performance of DynaFed, which relies solely on the synthetic dataset derived from the global model trajectory, proves inferior to FedAvg, demonstrating that the approximate information in synthetic datasets is insufficient for effective model training. For the CINIC-10 dataset, we mainly conduct experiments with partial client participation. The detailed results are given in Table \ref{tab:comparison3}. The experimental results also demonstrate that FedSynSAM outperforms all the baselines.

\begin{table}[t]
  \caption{Performance comparison of various methods under partial client participation under the \textit{Dir}(0.01) split strategy on the CINIC-10 dataset. The bold results indicate the best accuracy, and the underlined results indicate the second-best accuracy.}
  \label{tab:comparison3}
  \centering
  \small
  \setlength{\tabcolsep}{2pt}
  \begin{tabular}{lcccc}
    \toprule
    & \multicolumn{4}{c}{CINIC-10}  \\
    \cmidrule(lr){2-5} 
    & \multicolumn{2}{c}{Stochastic Quantization} & \multicolumn{2}{c}{Top-$k$}  \\
    \cmidrule(lr){2-3} \cmidrule(lr){4-5} 
    Method & 4bit & 8bit & 0.1 & 0.25  \\

    \midrule
    \multicolumn{5}{c}{\textbf{$\textit{Dir}(0.01)$, 50 clients, 20\% participation}} \\
    \midrule
    {\footnotesize FedAvg} & 55.58{\tiny  $\pm $ 0.31} & 55.42{\tiny  $\pm $ 0.30} & 49.53{\tiny  $\pm $ 0.28} & 52.78{\tiny  $\pm $ 0.38} \\
    {\footnotesize DynaFed} & 38.06{\tiny  $\pm $ 0.16} & 38.23{\tiny  $\pm $ 0.11} & 37.98{\tiny  $\pm $ 0.09} & 38.63{\tiny  $\pm $ 0.16} \\
    {\footnotesize FedSAM} & 55.27{\tiny  $\pm $ 0.37} & 55.18{\tiny  $\pm $ 0.28} & 48.99{\tiny  $\pm $ 0.38} & 52.63{\tiny  $\pm $ 0.27} \\
    {\footnotesize FedLESAM} & 55.38{\tiny  $\pm $ 0.34} & 55.46{\tiny  $\pm $ 0.41} & 49.13{\tiny  $\pm $ 0.28} & 52.62{\tiny  $\pm $ 0.34} \\
    {\footnotesize FedSMOO} & 53.28{\tiny  $\pm $ 0.70} & 54.58{\tiny  $\pm $ 0.30} & 49.82{\tiny  $\pm $ 0.05} & 52.08{\tiny  $\pm $ 0.09} \\
    {\footnotesize FedGAMMA} & 52.67{\tiny  $\pm $ 0.42} & 52.34{\tiny  $\pm $ 0.07} & 45.21{\tiny  $\pm $ 0.55} & 49.38{\tiny  $\pm $ 0.11} \\
    {\footnotesize FedLESAM-D} & 51.18{\tiny  $\pm $ 0.14} & 52.22{\tiny  $\pm $ 0.04} & 49.59{\tiny  $\pm $ 0.75} & 51.41{\tiny  $\pm $ 0.29} \\
    {\footnotesize FedLESAM-S} & \underline{56.46}{\tiny  $\pm $ 0.84} & \underline{56.86}{\tiny  $\pm $ 0.62} & \underline{50.71}{\tiny  $\pm $ 0.19} & \underline{54.79}{\tiny  $\pm $ 0.55} \\
    {\footnotesize FedSynSAM} & \textbf{56.95}{\tiny  $\pm $ 0.04} & \textbf{56.96}{\tiny  $\pm $ 0.05} & \textbf{52.16}{\tiny  $\pm $ 0.35} & \textbf{55.68}{\tiny  $\pm $ 0.08} \\

    \bottomrule
  \end{tabular}
\end{table}

\begin{figure}[htbp]
    \centering
    \begin{subfigure}[b]{0.24\textwidth}
        \centering
        \includegraphics[width=1\textwidth]{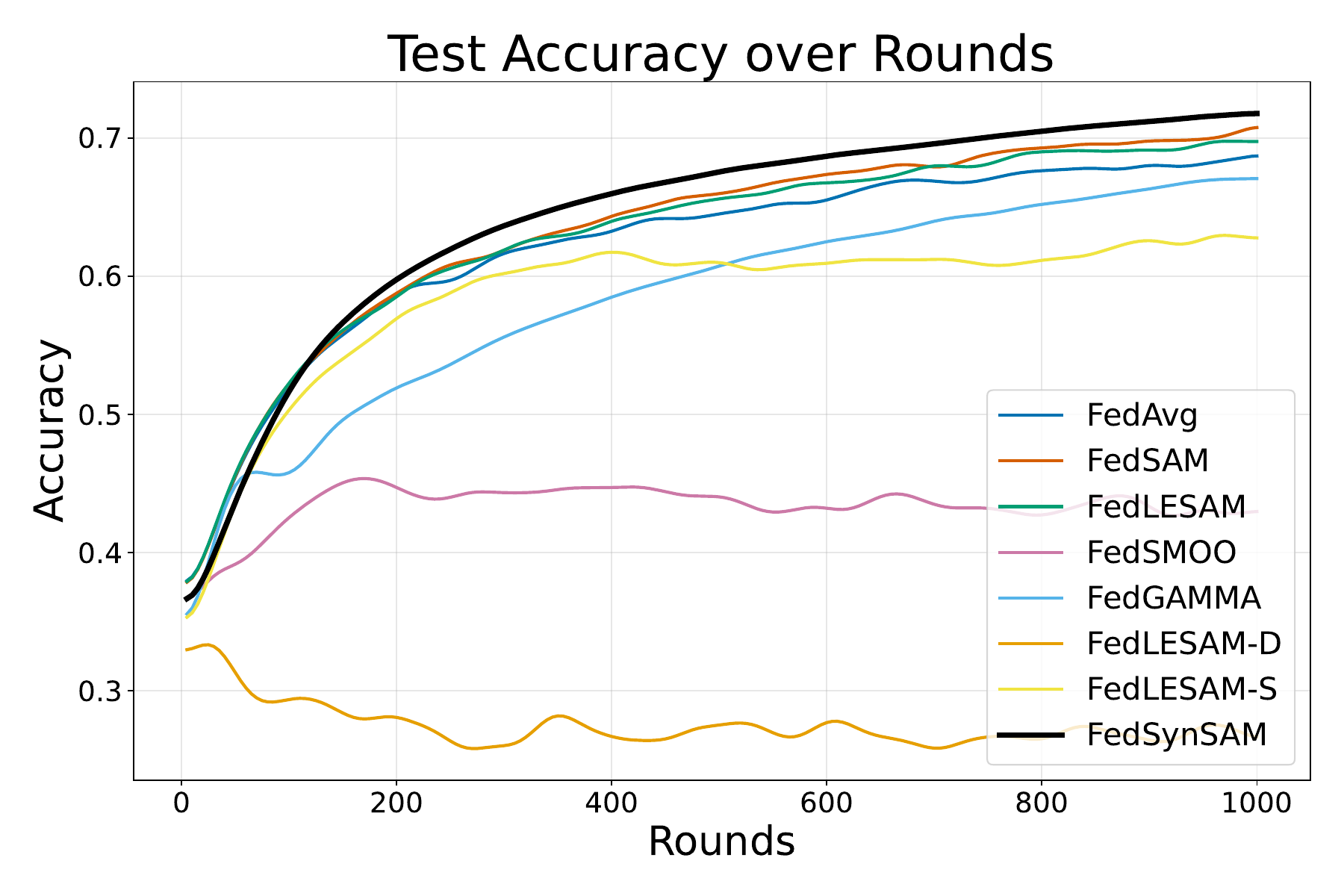}
        \caption{Stochastic Quantization,\\4bit}
    \end{subfigure}
    \hfill
    \begin{subfigure}[b]{0.24\textwidth}
        \centering
        \includegraphics[width=1\textwidth]{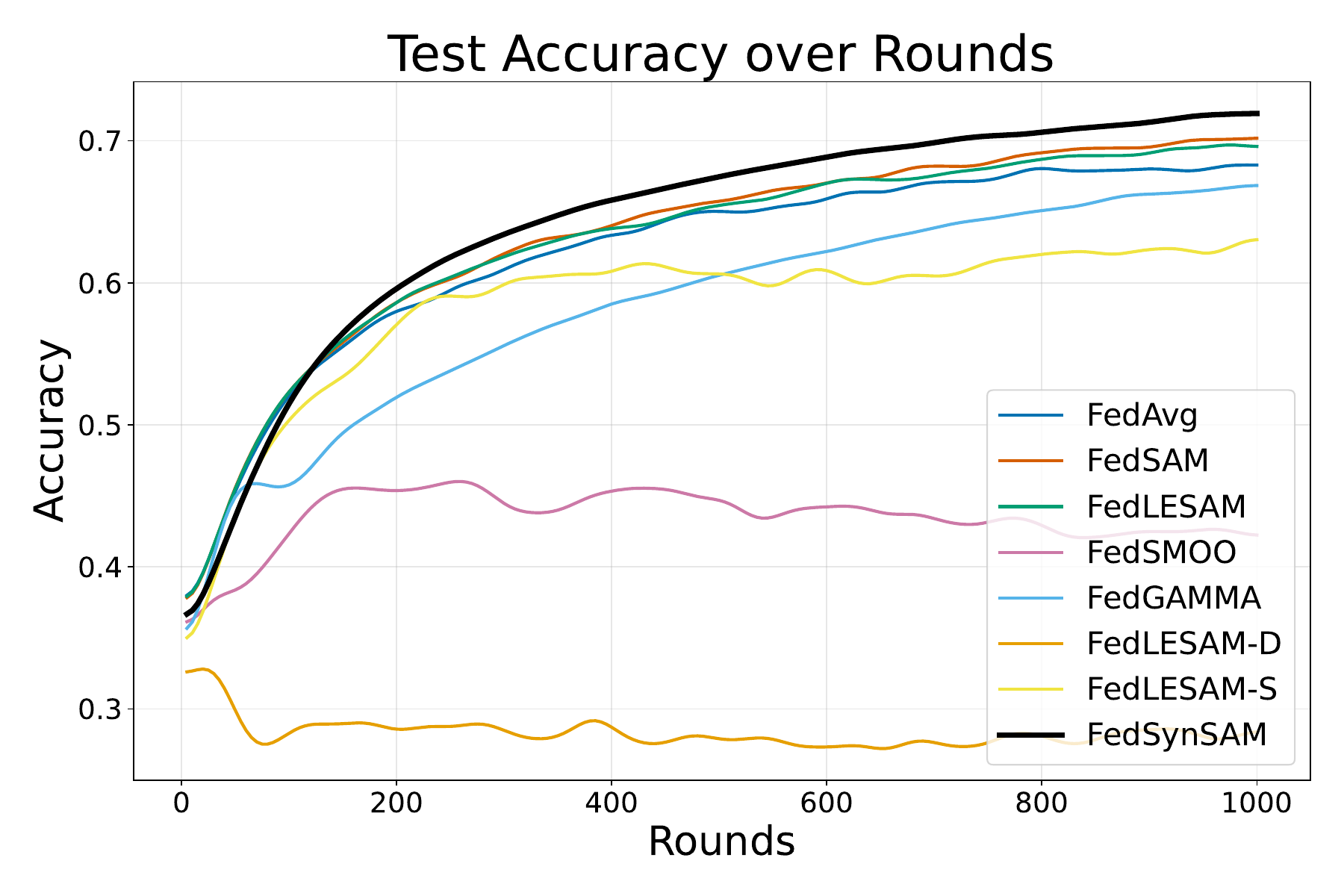}
        \caption{Stochastic Quantization,\\8bit}
    \end{subfigure}
    \hfill
    \begin{subfigure}[b]{0.24\textwidth}
        \centering
        \includegraphics[width=1\textwidth]{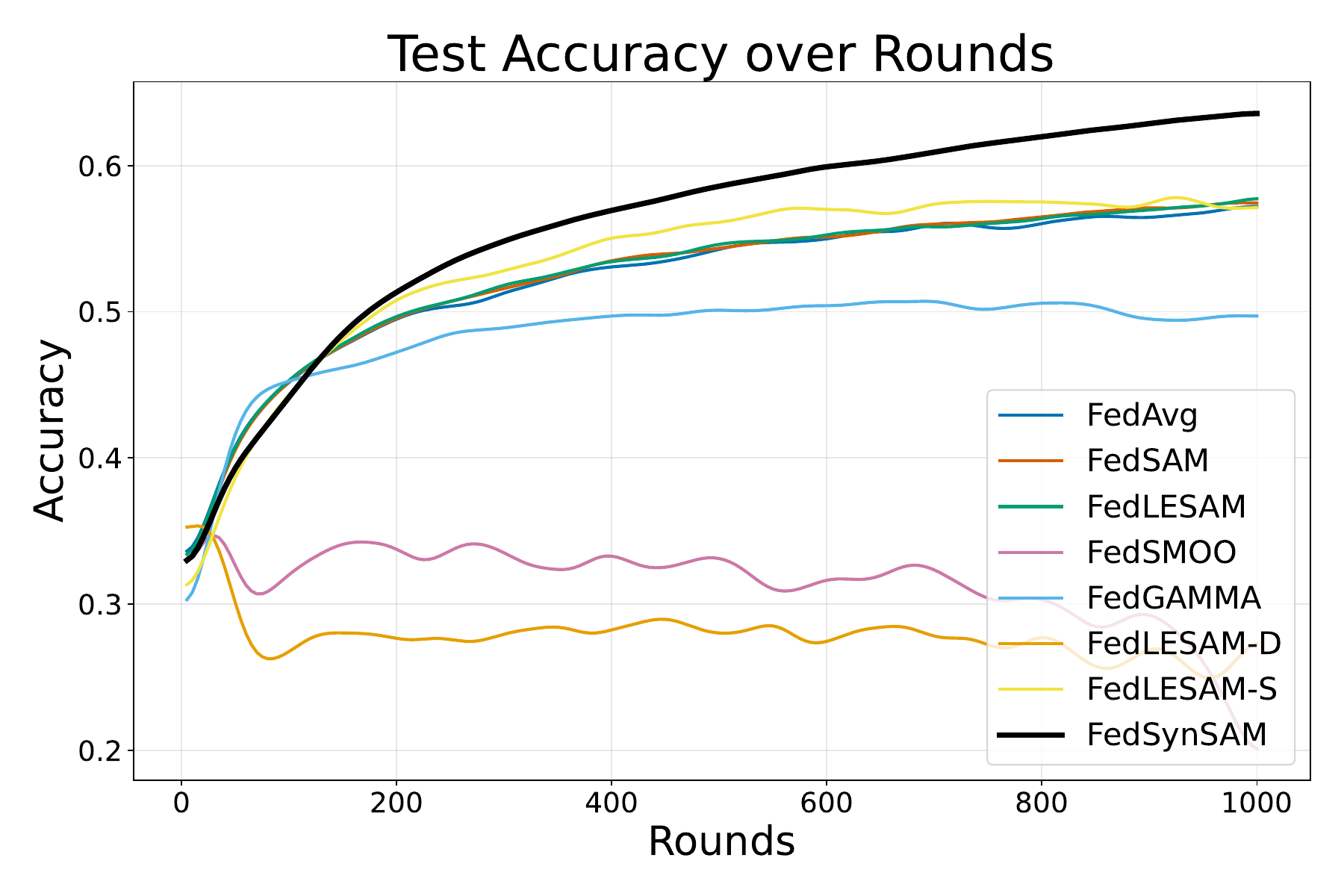}
        \caption{Top-$k$ 0.1}
    \end{subfigure}
    \hfill
    \begin{subfigure}[b]{0.24\textwidth}
        \centering
        \includegraphics[width=1\textwidth]{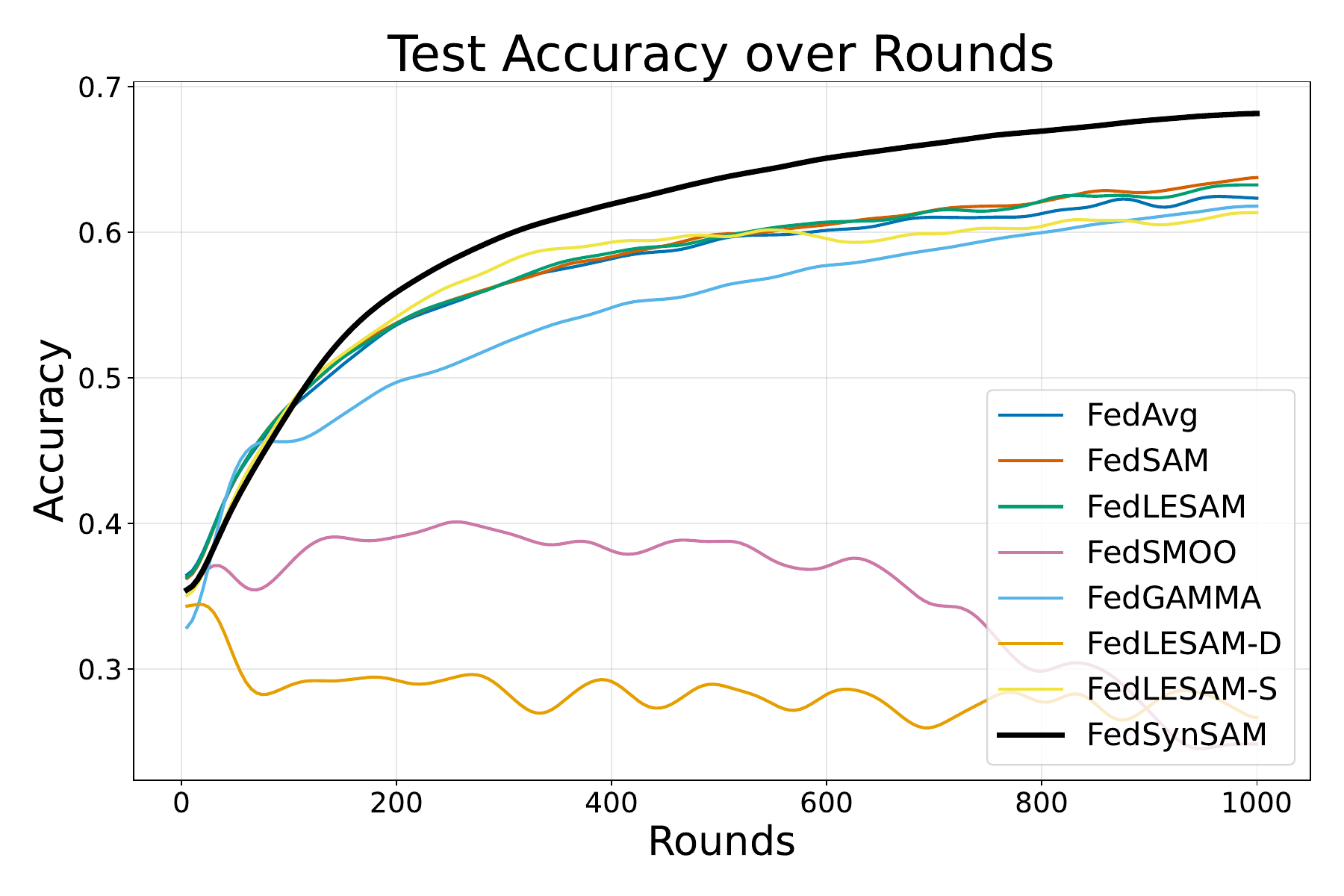}
        \caption{Top-$k$ 0.25}
    \end{subfigure}

    \caption{Comparison of test accuracy of training ConvNet on CIFAR-10 under the \textit{Path}(1) non-IID setting with different gradient compressors.}
    \label{fig:comparison}
\end{figure}

Since SAM-based methods (except FedLESAM) require twice the computation per iteration at the client side, we further conduct experiments on CIFAR-10 under equal computation budgets for a fair comparison with FedAvg and FedLESAM. Specifically, we compare FedAvg and FedLESAM with 20 local steps against FedSynSAM with 10 steps, and FedSynSAM with 5 steps against FedAvg and FedLESAM with 10 steps. The detailed results are presented in Table~\ref{tab:equal_comp}, showing that FedSynSAM consistently achieves the best performance under equal computation costs. 

\begin{table}[t]
  \caption{Performance comparison under equal computation cost. 
  (Top) FedAvg/FedLESAM (20 steps) vs. FedSynSAM (10 steps); 
  (Bottom) FedAvg/FedLESAM (10 steps) vs. FedSynSAM (5 steps). 
  The bold results indicate the best accuracy.}
  \label{tab:equal_comp}
  \centering
  \small
  \setlength{\tabcolsep}{2pt}
  \begin{tabular}{lcccc}
    \toprule
    & \multicolumn{2}{c}{Stochastic Quantization} & \multicolumn{2}{c}{Top-$k$}  \\
    \cmidrule(lr){2-3} \cmidrule(lr){4-5} 
    Method & 4bit & 8bit & 0.1 & 0.25  \\
    \midrule
    \multicolumn{5}{c}{\textbf{(Top) Path(1), non-IID, 10 clients, full participation}} \\
    \midrule
    {\footnotesize FedAvg} & 67.18{\tiny  $\pm $0.41} & 66.91{\tiny  $\pm $0.74} & 57.94{\tiny  $\pm $0.95} & 62.27{\tiny  $\pm $1.32} \\
    {\footnotesize FedLESAM} & 67.94{\tiny  $\pm $0.82} & 68.42{\tiny  $\pm $1.25} & 58.19{\tiny  $\pm $0.75} & 63.36{\tiny  $\pm $1.11} \\
    {\footnotesize FedSynSAM} & \textbf{71.56}{\tiny  $\pm $0.19} & \textbf{71.70}{\tiny  $\pm $0.32} & \textbf{63.60}{\tiny  $\pm $0.23} & \textbf{68.05}{\tiny  $\pm $0.18} \\
    \midrule
    \multicolumn{5}{c}{\textbf{(Top) $\textit{Dir}(0.01)$, non-IID, 50 clients, 20\% participation}} \\
    \midrule
    {\footnotesize FedAvg} & 71.20{\tiny  $\pm $0.21} & 71.49{\tiny  $\pm $0.41} & 67.74{\tiny  $\pm $0.31} & 70.89{\tiny  $\pm $0.15} \\
    {\footnotesize FedLESAM} & 72.15{\tiny  $\pm $0.42} & 72.27{\tiny  $\pm $0.54} & 68.37{\tiny  $\pm $0.38} & 70.87{\tiny  $\pm $0.37} \\
    {\footnotesize FedSynSAM} & \textbf{72.74}{\tiny  $\pm $0.02} & \textbf{72.96}{\tiny  $\pm $0.12} & \textbf{70.53}{\tiny  $\pm $0.18} & \textbf{72.57}{\tiny  $\pm $0.36} \\
    \midrule
    \multicolumn{5}{c}{\textbf{(Bottom) Path(1), non-IID, 10 clients, full participation}} \\
    \midrule
    {\footnotesize FedAvg} & 68.55{\tiny  $\pm $0.38} & 68.38{\tiny  $\pm $0.29} & 57.50{\tiny  $\pm $0.74} & 62.72{\tiny  $\pm $0.71} \\
    {\footnotesize FedLESAM} & 69.69{\tiny  $\pm $0.40} & 69.70{\tiny  $\pm $0.31} & 57.97{\tiny  $\pm $0.35} & 63.49{\tiny  $\pm $0.36} \\
    {\footnotesize FedSynSAM} & \textbf{71.77}{\tiny  $\pm $0.22} & \textbf{71.84}{\tiny  $\pm $0.11} & \textbf{62.42}{\tiny  $\pm $0.26} & \textbf{67.43}{\tiny  $\pm $0.05} \\
    \midrule
    \multicolumn{5}{c}{\textbf{(Bottom) $\textit{Dir}(0.01)$, non-IID, 50 clients, 20\% participation}} \\
    \midrule
    {\footnotesize FedAvg} & 71.32{\tiny  $\pm $0.05} & 71.47{\tiny  $\pm $0.03} & 67.11{\tiny  $\pm $0.06} & 69.67{\tiny  $\pm $0.13} \\
    {\footnotesize FedLESAM} & 71.72{\tiny  $\pm $0.21} & 71.67{\tiny  $\pm $0.12} & 67.07{\tiny  $\pm $0.45} & 69.71{\tiny  $\pm $0.06} \\
    {\footnotesize FedSynSAM} & \textbf{72.23}{\tiny  $\pm $0.34} & \textbf{72.27}{\tiny  $\pm $0.51} & \textbf{67.16}{\tiny  $\pm $0.13} & \textbf{70.11}{\tiny  $\pm $0.21} \\
    \bottomrule
  \end{tabular}
\end{table}

\textbf{Loss landscape visualization:} As illustrated in Figure \ref{fig:fig_results}, FedSynSAM exhibits a notably flatter loss landscape compared to FedAvg, FedSAM, and FedLESAM when adopting the stochastic quantization compressor, which corroborates the results in Table \ref{tab:comparison}.

\begin{figure*}
    \centering
    \begin{subfigure}[b]{0.22\textwidth}
        \centering
        \scalebox{1}[0.7]{\includegraphics[width=1\textwidth, trim=1cm 1.2cm 1.2cm 1.2cm, clip]{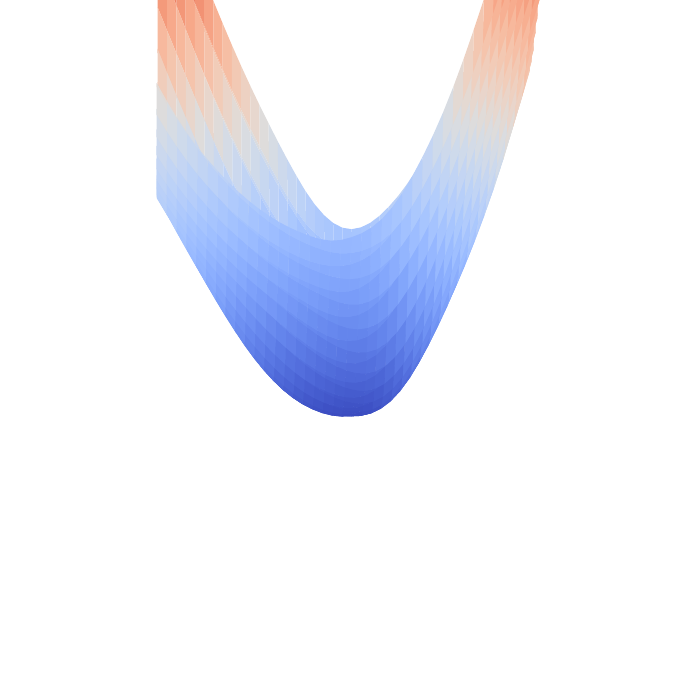}}
        \vspace{-0.3in}
        \caption{FedAvg}
        \label{fig:image11}
    \end{subfigure}
    \hfill
    \begin{subfigure}[b]{0.22\textwidth}
        \centering
        \scalebox{1}[0.7]{\includegraphics[width=1\textwidth, trim=1cm 1.2cm 1.2cm 1.2cm, clip]{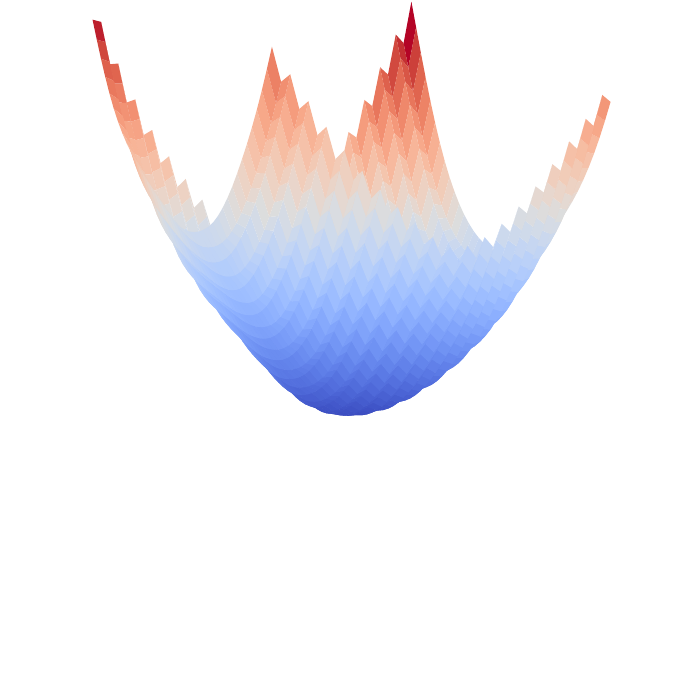}}
        \vspace{-0.3in}
        \caption{FedSAM}
        \label{fig:image12}
    \end{subfigure}
    \hfill
    \begin{subfigure}[b]{0.22\textwidth}
        \centering
        \scalebox{1}[0.7]{\includegraphics[width=1\textwidth, trim=1cm 1.2cm 1.2cm 1.2cm, clip]{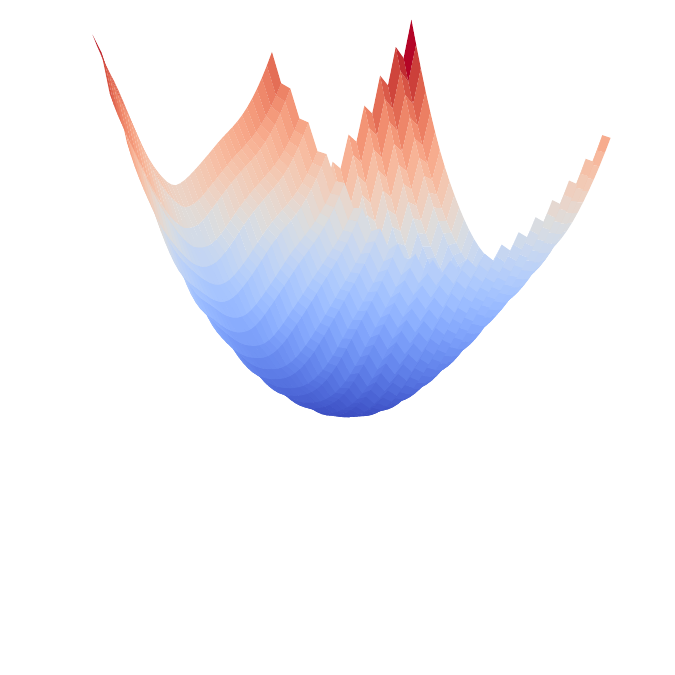}}
        \vspace{-0.3in}
        \caption{FedLESAM}
        \label{fig:image13}
    \end{subfigure}
    \hfill
    \begin{subfigure}[b]{0.22\textwidth}
        \centering
        \scalebox{1}[0.7]{\includegraphics[width=1\textwidth, trim=1cm 1.2cm 1.2cm 1.2cm, clip]{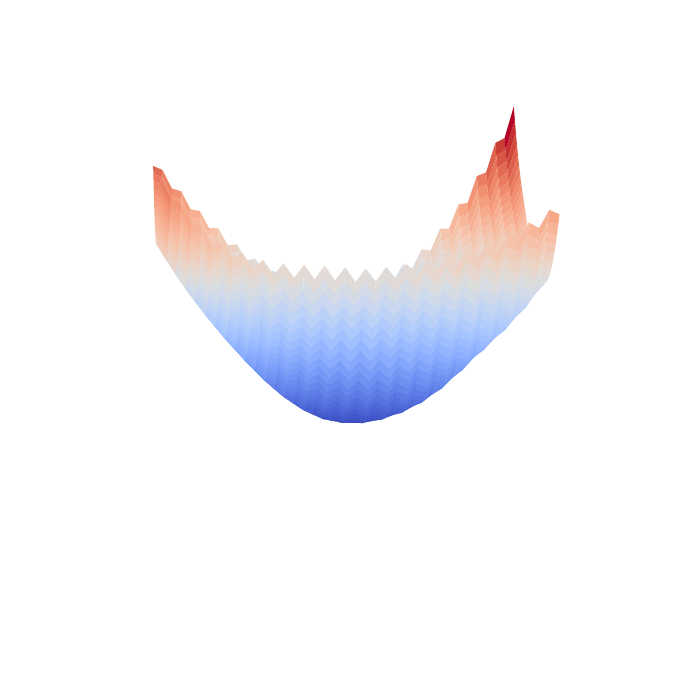}}
        \vspace{-0.3in}
        \caption{FedSynSAM}
        \label{fig:image14}
    \end{subfigure}
    
    \caption{Loss landscape of FedAvg, FedSAM, FedLESAM, and FedSynSAM with 4-bit stochastic quantization under the \textit{Path}(1) data distribution on CIFAR-10 with full participation of 10 clients.}
    \label{fig:fig_results}
    \vspace{-0.1in}

\end{figure*}

\begin{figure}[h]
    \centering
    \includegraphics[width=0.4\textwidth]{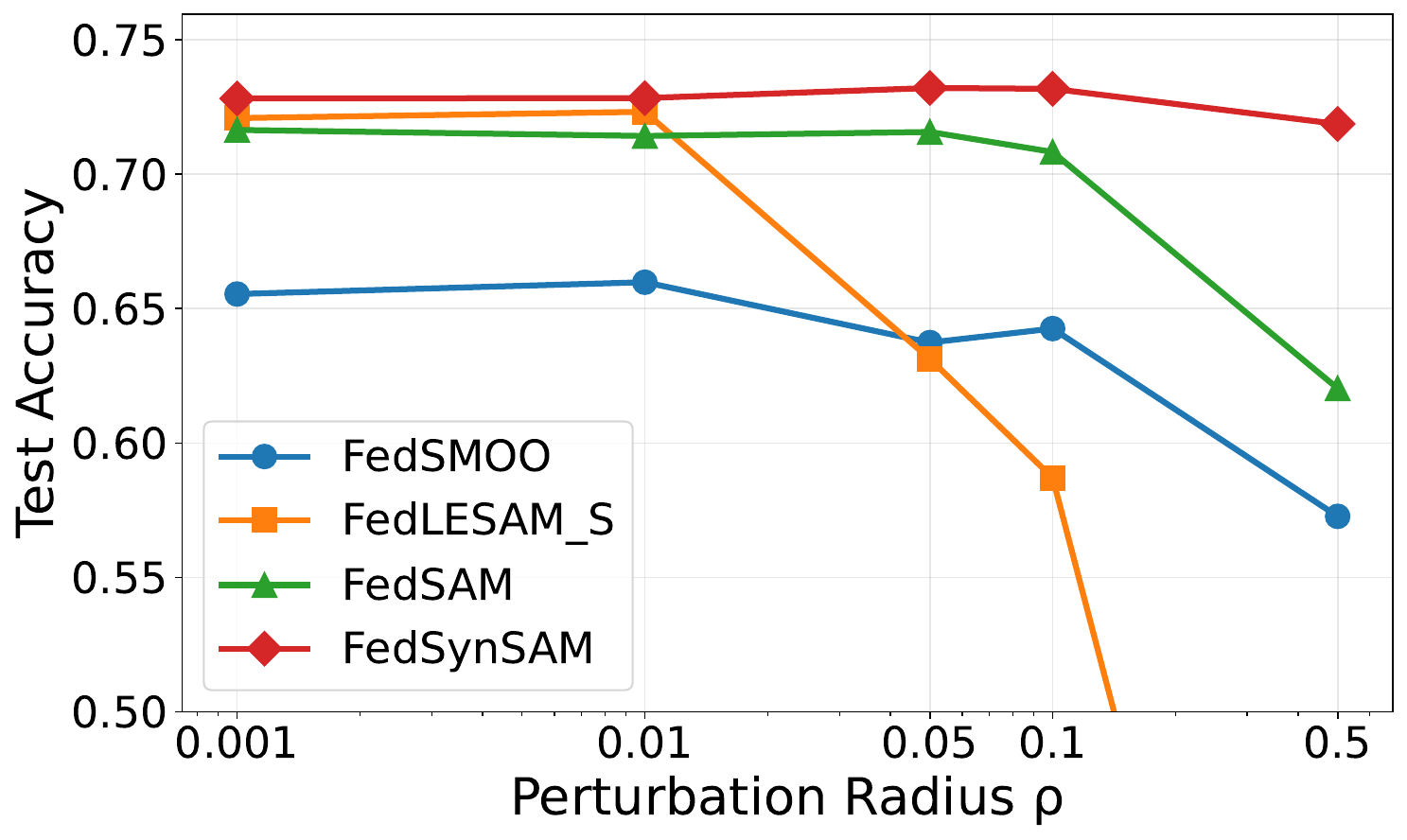}
    \caption{Impact of perturbation radius $\rho$ on CIFAR-10 with partial participation and no compression.}
    \label{fig:ablation}
\end{figure}

\subsection{Ablation Study}
In this section, we present ablation studies to examine the impacts of key hyperparameters. Unless otherwise stated, the experiments are conducted on CIFAR-10 with 4-bit stochastic quantization and full participation of 10 clients under the \textit{Path}(1) non-IID setting.

\textbf{Impact of Synthetic Dataset Size:}
Table~\ref{tab:ablation_dataset_size} shows that increasing the number of images per class slightly improves accuracy. Notably, using only 20 images per class already achieves strong performance (71.56\%), and enlarging the dataset to 40 images per class brings only marginal gains. These results confirm that remarkably compact synthetic datasets are sufficient to obtain competitive accuracy, consistent with prior findings on dataset condensation~\cite{cazenavette2022dataset,zhao2020dataset}.

\begin{table}[t]
  \centering
  \caption{Impact of synthetic dataset size (images per class) on CIFAR-10.}
  \label{tab:ablation_dataset_size}
  \small
  \begin{tabular}{lcccc}
    \toprule
    Images/class & 10 & 20 & 30 & 40 \\
    \midrule
    FedSynSAM & 71.16{\tiny $\pm$0.35} & 71.56{\tiny $\pm$0.19} & 72.27{\tiny $\pm$0.08} & 72.80{\tiny $\pm$0.18} \\
    \bottomrule
  \end{tabular}
\end{table}

\begin{table}[h]
  \centering
  \caption{Impact of the number of initial rounds $R$ on CIFAR-10.}
  \label{tab:ablation_R}
  \small
  \begin{tabular}{lccccc}
    \toprule
    $R$ & 20 & 30 & 50 & 200 \\
    \midrule
    FedSynSAM & 71.51{\tiny $\pm$0.15} & 71.56{\tiny $\pm$0.19} & 
    71.67{\tiny $\pm$0.18} & 71.58{\tiny $\pm$0.35} \\
    \bottomrule
  \end{tabular}
\end{table}

\begin{table}[!htbp]
  \centering
  \caption{Sensitivity of FedSynSAM to $\eta_x$ and $\eta_\alpha$ on CIFAR-10.}
  \label{tab:ablation_lr}
  \small
  \begin{tabular}{lccc}
    \toprule
    $\eta_\alpha \backslash \eta_x$ & 100 & 1000 & 10000 \\
    \midrule
    $1 \times 10^{-6}$ & 71.60{\tiny $\pm$0.20} & 71.63{\tiny $\pm$0.29} & 70.35{\tiny $\pm$0.26} \\
    $1 \times 10^{-5}$ & 71.52{\tiny $\pm$0.13} & 71.56{\tiny $\pm$0.19} & 70.65{\tiny $\pm$0.35} \\
    $1 \times 10^{-4}$ & 71.01{\tiny $\pm$0.08} & 71.52{\tiny $\pm$0.23} & 71.42{\tiny $\pm$0.33} \\
    \bottomrule
  \end{tabular}
\end{table}

\textbf{Impact of Initial Rounds $R$:}
The parameter $R$ controls the number of initial client updates used to synthesize informative gradients. Table~\ref{tab:ablation_R} shows that the performance of FedSynSAM is robust to this choice. 

\textbf{Sensitivity to Learning Rates $\eta_x$ and $\eta_\alpha$:}
We further examine the sensitivity of the synthetic data optimization to the learning rates $\eta_x$ and $\eta_\alpha$. As shown in Table~\ref{tab:ablation_lr}, FedSynSAM is stable across a wide range of settings. Once $\eta_x$ is properly tuned, the choice of $\eta_\alpha$ has limited impact, as it adaptively searches for an appropriate step size. This demonstrates that the algorithm is not overly sensitive to these hyperparameters, supporting its robustness and reproducibility.

\textbf{Impact of Perturbation Radius $\rho$:}
The perturbation radius $\rho$ is a key parameter in SAM-based methods. We compare FedSynSAM with FedSMOO and FedLESAM-S, representative dynamic-regularizer and control variate-based approaches, on CIFAR-10 with partial client participation and no gradient compression. As illustrated in Figure~\ref{fig:ablation}, FedSynSAM consistently achieves the highest accuracy without incurring any additional communication overhead. Moreover, it exhibits greater robustness to the choice of $\rho$, which validates the effectiveness of the proposed approach under practical settings.


\section{Conclusion}\label{conclusion}
\noindent In this work, we found that gradient compression in FL may result in sharper loss landscapes, and integrated SAM into local training to address this challenge. To mitigate the data heterogeneity issue, we developed a novel approach that leverages the global model trajectory to enhance the estimation of the global perturbation. The convergence of the proposed algorithm is established, and extensive experimental results validate its effectiveness. The proposed method demonstrates significant potential for application in bandwidth-constrained and heterogeneous scenarios.


\bibliography{IEEEabrv,ref}
\bibliographystyle{IEEETran}

\appendices
\section{Convergence Analysis}
\subsection{Technical Preliminaries}
\noindent \textbf{Additional notations.} For each round $t=0,1,\cdots ,T-1$, we denote 
\begin{align}\label{defofx}
    \bar{\bm{w}}^{t+1} &= \bm{w}^{t} + \frac{\eta_g}{S}\sum_{i \in \S^t}\Delta_i^t,\\
    \tbw^{t+1}&=\bar{\bm{w}}^{t+1}+\rho\frac{\nabla F(\bar{\bm{w}}^{t+1})}{\|\nabla F(\bar{\bm{w}}^{t+1})\|},\\
    \Delta^t&=\frac{1}{S}\sum_{i \in \S^t}\Delta_i^t,
\end{align}
where $\bar{\bm{w}}^{t+1}$ represents updated global model without gradient compression and $\Delta^t$ is the aggregated model updates without gradient compression.
\begin{Lemma}\label{Lemma:triangleinequality}
		(Relaxed triangle inequality). Let $\{\bm{v}_1 , \dots, \bm{v}_\tau \}$ be $\tau$ vectors in $\mathbb{R}^d$. Then, the following are true: (1) $\|\bm{v}_i + \bm{v}_j \|^2 \leq (1+a)\|\bm{v}_i \|^2 + (1+\frac{1}{a})\|\bm{v}_j \|^2$ for any $a > 0$, and (2) $\|\sum_{i=1}^{\tau} \bm{v}_i \|^2 \leq \tau \sum_{i=1}^{\tau}\|\bm{v}_i \|^2$.
	\end{Lemma}
	
	\begin{Lemma}\label{Lemma:dependent}
		For random variables $\bm{x}_1 , \dots, \bm{x}_n$, we have
		\begin{equation}
			\E [\|\bm{x}_1 + \cdots + \bm{x}_n \|^2 ] \leq n\E [\|\bm{x}_1 \|^2 + \cdots + \|\bm{x}_n \|^2 ].
		\end{equation}
	\end{Lemma}
	
	\begin{Lemma}\label{Lemma:independent}
		For independent, mean $0$ random variables $\bm{x}_1 , \dots, \bm{x}_n$, we have
		\begin{equation}
			\E [\|\bm{x}_1 + \cdots + \bm{x}_n \|^2 ] = \E [\|\bm{x}_1 \|^2 + \cdots + \|\bm{x}_n \|^2 ].
		\end{equation}
	\end{Lemma}

\begin{Lemma}\label{Lemma:full_init}
    If Assumption \ref{ass:smooth}, \ref{ass:quant_var} hold, we have
\begin{equation}
    \mathbb{E}[F(\bm{w}^{t+1})] \leq \mathbb{E}[F(\bar{\bm{w}}^{t+1})]+\frac{L}{2}\mathbb{E}\Vert \bm{w}^{t+1} - \bar{\bm{w}}^{t+1}\Vert_2^2,
\end{equation}
for any round $t=0, 1, \cdots, T-1$.
\end{Lemma}
\noindent\textit{Proof.} Recall that for any $L$-smooth function $F$ we have
\begin{equation}
    F(\bm{x}) \leq F(\bm{y})+\langle \nabla F(\bm{y}), \bm{x}-\bm{y} \rangle + \frac{L}{2}\Vert \bm{x}-\bm{y} \Vert_2^2 .
\end{equation}
Therefore, we can write
\begin{equation}
\begin{split}
        F(\bm{w}^{t+1})&=F(\bar{\bm{w}}^{t+1} + \bm{w}^{t+1} - \bar{\bm{w}}^{t+1}) \\&\leq F(\bar{\bm{w}}^{t+1})+\langle \nabla F(\bar{\bm{w}}^{t+1}), \bm{w}^{t+1} - \bar{\bm{w}}^{t+1} \rangle\\&~~~~~+\frac{L}{2}\Vert \bm{w}^{t+1} - \bar{\bm{w}}^{t+1}\Vert_2^2.
\end{split}
\end{equation}
Taking expectation on both sides completes the proof.\hfill~$\Box$

\begin{Lemma}\label{Lemma:deltadrift}
    (\cite{qu2022generalized}) Suppose function $F$ satisfies Assumptions~\ref{ass:smooth}-\ref{ass:sigmag}. Then, the updates for any learning rate satisfying $\eta_l \leq \frac{1}{4KL}$ have the drift due to $\bm{\epsilon}_{i,k} - \bm{\epsilon}$:
		\begin{equation}
			\mathcal{E}_{\epsilon} = \frac{1}{N}\sum_{i}\E [\|\bm{\epsilon}_{i,k} - \bm{\epsilon} \|^2 ] \leq 2K^2 L^2 \eta_l^2 \rho^2 , 
		\end{equation}
where the definitions of $\bm{\epsilon}$ and $\bm{\epsilon}_{i,k}$ are as follows:
	\begin{equation}
		\bm{\epsilon} = \rho \frac{\nabla F(\bm{w})}{\|\nabla F(\bm{w})\|}, ~~~ \bm{\epsilon}_{i,k} = \rho \frac{\nabla F (\bm{w}_{i,k} ,\xi_i )}{\|\nabla F (\bm{w}_{i,k}, \xi_i )\|}. 
	\end{equation}
\end{Lemma}

\begin{Lemma}\label{Lemma:xdrift}
		(\cite{qu2022generalized}) Suppose functions $F$ satisfies Assumptions~\ref{ass:smooth}-\ref{ass:sigmag}. Then, the updates for any learning rate satisfying $\eta_l \leq \frac{1}{10KL}$ have the drift due to $\bm{w}_{i,k} - \bm{w}$:
		\begin{equation}
        \begin{split}
           			\mathcal{E}_{w} = \frac{1}{N}&\sum_{i}\E [\|\bm{w}_{i,k} - \bm{w} \|^2 ] \leq 5K\eta_l^2 (2L^2 \rho^2 \sigma_l^2 \\&+ 6K\gamma + 6K\|\nabla F(\tw)\|^2 ) + 24K^3 \eta_l^4 L^4 \rho^2 .  
        \end{split}
		\end{equation}
	\end{Lemma}
    
	\begin{Lemma}\label{lemma:SAMvariance}
		(\cite{qu2022generalized}) The stochastic gradient $\tg_{i,k}^t=\nabla F_i (\tw_{i,k}^t,\xi_{i,k}^t )$ computed at the $i$-th client using a random local data sample $\xi_{i,k}^t$ is an unbiased estimator of $\nabla F_i (\tw_{i,k}^t)$ and satisfies
		\begin{equation}
			\begin{split}
				\E\bigg[\bigg\|\sum_{k=0}^{K-1} \tg_{i,k}^t \bigg\|^2 \bigg]\leq K\sum_{k=0}^{K-1} \E [\|\nabla F_i (\tw_{i,k}^t )\|^2 ] +\frac{K L^2\rho^2}{N}\sigma_l^2 ,\\
				\E \bigg[\bigg\|\sum_{k=0}^{K-1} \tg_{i,k}^t \bigg\|^2 \bigg] \leq K\sum_{k=0}^{K-1} \E [\|\nabla F_i (\tw_{i,k}^t ) \|^2 ] + KL^2 \rho^2 \sigma_l^2 .
			\end{split}
		\end{equation}
	\end{Lemma}

\subsection{Proof of Lemma \ref{lemma:sigmag}}
\begin{equation}
    \begin{split}
        \|\nabla F_i (&\bm{w} + \hat{\bm{\epsilon}}_i ) - \nabla F(\bm{w} + \bm{\epsilon})\|^2\\
        &\overset{\text{(a)}}{\leq} 2\|\nabla F_i(\bm{w}+ \hat{\bm{\epsilon}}_i )-\nabla F(\bm{w}+\hat{\bm{\epsilon}}_i)\|^2\\&~~~~+2\|\nabla F(\bm{w}+\hat{\bm{\epsilon}}_i)-\nabla F(\bm{w}+\bm{\epsilon})\|^2\\
        &\overset{\text{(b)}}{\leq}2\sigma_g^2+2 L^2\|\hat{\bm{\epsilon}}_i-\bm{\epsilon}\|^2=2\sigma_g^2+4 L^2\rho^2(1-\cos{\theta}),
    \end{split}
\end{equation}
    where (a) is from Lemma \ref{Lemma:dependent}; (b) is from Assumption \ref{ass:smooth} and \ref{ass:sigmag}, $\theta=\arccos{\big(\frac{\nabla \hat{F}_i(\bm{w})\cdot \nabla F(\bm{w})}{\|\nabla \hat{F}_i(\bm{w})\|\| \nabla F(\bm{w})\|}\big)}$.$\hfill \Box$

\subsection{Proof of Theorem \ref{full_ca}}
In the full participation case, $S=N$.

\begin{Lemma}\label{Lemma:d2}
		(\cite{qu2022generalized}) For the full client participation scheme, we can bound $\E [\|\Delta^t \|^2 ]$ as follows:
		\begin{equation}
			\E_t [\|\Delta^t \|^2 ] \leq \frac{K\eta_l^2 L^2 \rho^2}{N}\sigma_l^2 + \frac{\eta_l^2}{N^2}\bigg[\bigg\|\sum_{i,k}\nabla F_i(\tw_{i,k}^t ) \bigg\|^2 \bigg]. 
		\end{equation}
	\end{Lemma}


\begin{Lemma}\label{Lemma:A1}(\cite{qu2022generalized})
		\begin{equation}
        \begin{split}
			\langle &\nabla F(\tw^t ),\E_t [-\Delta^t + \eta_l K\nabla F(\tw^t ) ]\rangle \leq \frac{\eta_l K}{2}\|\nabla F(\tw^t ))\|^2 \\ &+ K \eta_l L^2 \mathcal{E}_{w} + K \eta_l L^2 \mathcal{E}_{\epsilon}- \frac{\eta_l}{2KN^2}\E_t \bigg\|\sum_{i,k} \nabla F_i (\tw_{i,k}) \bigg\|^2 .
        \end{split}
		\end{equation}
	\end{Lemma}

\begin{Lemma}\label{Lemma:decentsam}
		For all $t \in T-1$ and $i \in \S^t$, with the choice of learning rate $\eta_g\eta_l \leq \frac{1}{KL}$, the iterates generated satisfy
		\begin{equation}
			\begin{split}
				\E_t &[F(\tbw^{t+1})]  \leq  F(\tw^t ) \\&- K\eta_g \eta_l \bigg(\frac{1}{2} - 30K^2 L^2 \eta_l^2 \bigg)\|\nabla F(\tw^t )\|^2 \\&+ K\eta_g \eta_l (10KL^4 \eta_l^2 \rho^2 \sigma_l^2 + 90K^2 L^2 \eta_l^2 \sigma_g^2 \\&+ 180 K^2 L^4 \eta_l^2 \rho^2 
				+ 120 K^4 L^6 \eta_l^6 \rho^2 + 16K^3 \eta_l^4 L^6 \rho^2 \\& + \frac{\eta_g \eta_l L^3 \rho^2}{N}\sigma_l^2 ) , 
			\end{split}
		\end{equation}
		where the expectation is w.r.t. the stochasticity of the algorithm.
	\end{Lemma}
\noindent\textit{Proof.} 
Due to the smoothness in Assumption~\ref{ass:smooth}, taking expectation of $F(\tbw^{t+1})$ over the randomness at communication round $t$, we have
	\begin{equation}\label{Eq:fedsam}
		\begin{split}
			& \E_{t} [F(\tbw^{t+1})]  \leq F(\tw^t ) + \E_t [\langle \nabla F(\tw^t ), \tbw^{t+1} - \tw^t \rangle] \\&~~~~~~+\frac{L}{2}\E_t [\|\tbw^{t+1} - \tw^t \|^2 ]\\
			& \overset{\text{(a)}}{=} F(\tw^t ) + \E_t \langle \nabla F(\tw^t ) ,  -\Delta^t + K \eta_g\eta_l \nabla F(\tw^t ) \\&~~~~~~- K \eta_g\eta_l \nabla F(\tw^t )\rangle + \frac{L\eta_g^2}{2} \E_t [\|\Delta^t \|^2]\\
			& \overset{\text{(b)}}{=} F(\tw^t ) - K \eta_g\eta_l \|\nabla F(\tw^t )\|^2 + \frac{L\eta_g^2}{2} \E_t [\|\Delta^t \|^2 ]\\&~~~~~~+ \langle \nabla F(\tw^t ) ,\E_t [-\Delta^t + K\eta_g\eta_l \nabla F(\tw^t ) ]\rangle \\
			& \overset{\text{(c)}}{\leq} F(\tw^t ) - \frac{K\eta_g \eta_l}{2} \|\nabla F(\tw^t )\|^2 + K \eta_g\eta_l L^2 (\mathcal{E}_{w} + \mathcal{E}_{\epsilon}) \\&~~~~- \frac{\eta_g \eta_l}{2KN^2}\E_t \bigg[\bigg\|\sum_{i,k}\nabla F_i (\tw_{i,k}^t )\bigg\|^2 \bigg] + \frac{L\eta_g^2}{2} \E_t [\|\Delta^t \|^2 ]\\
			& \overset{\text{(d)}}{\leq} F(\tw^t ) - \frac{K \eta_g\eta_l}{2} \|\nabla F(\tw^t )\|^2 + K \eta_g\eta_l L^2 \mathcal{E}_{w} \\
            &~~~~+ K \eta_g\eta_l L^2 \mathcal{E}_{\epsilon} + \frac{K \eta_g^2 \eta_l^2 L^3 \rho^2}{2N}\sigma_l^2 \\
            &~~~~+( \frac{L \eta_g^2 \eta_l^2}{2 N^2} - \frac{ \eta_g\eta_l}{2KN^2})\E_t \bigg[\bigg\|\sum_{i,k}\nabla F_i (\tw_{i,k}^t )\bigg\|^2 \bigg] \\
			& \overset{\text{(e)}}{\leq} F(\tw^t ) - K\eta_g \eta_l \bigg(\frac{1}{2} - 30K^2 L^2 \eta_l^2 \bigg)\|\nabla F(\tw^t )\|^2 \\&~~~~-+ K \eta_g \eta_l (10KL^4 \eta_l^2 \rho^2 \sigma_l^2 + 30K^2 L^2 \eta_l^2 \gamma \\&~~~~ 
			+ 120 K^4 L^6 \eta_l^6 \rho^2 + 16K^3 \eta_l^4 L^6 \rho^2 + \frac{ \eta_g\eta_l L^3 \rho^2}{N}\sigma_l^2 ) ,
		\end{split}
	\end{equation}
	where (a) is from the iterate update; (b) results from the unbiased estimators; (c) is from Lemma~\ref{Lemma:A1}; (d) is from Lemma~\ref{Lemma:d2} 
 and due to the fact that $\eta_g\eta_l \leq \frac{1}{KL}$
 and (e) is from Lemmas~\ref{Lemma:deltadrift} and \ref{Lemma:xdrift}. \hfill~$\Box$

Based on the lemmas established above, we present the proof of Theorem \ref{full_ca}.

\noindent\textit{Proof.} According to the definition in (\ref{defofx}) and use Assumption \ref{ass:quant_var}, we have
     \begin{equation}
     \begin{split}
         \mathbb{E}\Vert \bm{w}^{t+1} - \bar{\bm{w}}^{t+1} \Vert_2^2 &=\frac{\eta_g^2}{N^2}\mathbb{E}\|\sum_i\big(Q(\Delta_i^t)-\Delta_i^t\big)\|^2  \\
         &\overset{\text{(a)}}{=}\frac{\eta_g^2}{N^2}\sum_i\mathbb{E}\|\big(Q(\Delta_i^t)-\Delta_i^t\big)\|^2
         \\&\overset{\text{(b)}}{\leq} \frac{\eta_g^2}{N^2} \sum_i q\mathbb{E}\Vert \bm{w}^t_{i,K}-\bm{w}^t \Vert_2^2,
    \end{split}
     \end{equation}
where (a) is from the unbiasedness in Assumption \ref{ass:quant_var} and Lemma \ref{Lemma:independent}; (b) is from the variance assumption in Assumption \ref{ass:quant_var}.
According to Lemma \ref{Lemma:xdrift}, 
\begin{equation}\label{w_norm2}
\begin{split}
        \mathbb{E}\Vert \bm{w}^{t+1} - \bar{\bm{w}}^{t+1} \Vert_2^2 & \leq \frac{\eta_g^2q}{N} (5K\eta_l^2 (2L^2 \rho^2 \sigma_l^2 + 6K\gamma \\&+ 6K\|\nabla F(\tw)\|^2 ) + 24K^3 \eta_l^4 L^4 \rho^2).
\end{split}
\end{equation}
Further combining Lemma~\ref{Lemma:decentsam} and Lemma~\ref{Lemma:full_init}, and borrowing the approximation $\mathbb{E}[F(\bm{w}^{t+1})]=\mathbb{E}[F(\tilde{\bm{w}}^{t+1})]$ from \cite{qu2022generalized,fan2024locally} yields 
\begin{equation}\label{mydecay}
\begin{split}
    \mathbb{E}&[F(\bm{w}^{t+1})]=\mathbb{E}[F(\tilde{\bm{w}}^{t+1})] \leq F(\tw^t ) \\&- K \eta_g\eta_l \bigg(\frac{1}{2} - 30K^2 L^2 \eta_l^2  \bigg)\|\nabla F(\tw^t )\|^2 \\&+ K\eta_g \eta_l (10KL^4 \eta_l^2 \rho^2 \sigma_l^2 + 30K^2 L^2 \eta_l^2 \gamma 
				+ 120 K^4 L^6 \eta_l^6 \rho^2 \\&~~~~~~~~~~~~~~~~+ 16K^3 \eta_l^4 L^6 \rho^2 + \frac{\eta_g \eta_l L^3 \rho^2}{N}\sigma_l^2 )\\&+ \frac{\eta_g^2q}{N} (5K\eta_l^2 (2L^2 \rho^2 \sigma_l^2 + 6K\gamma 
                + 6K\|\nabla F(\tw)\|^2 ) \\&~~~~~~~~~~~~~~~~++ 24K^3 \eta_l^4 L^4 \rho^2).
\end{split}
\end{equation}
For full client participation, summing the result of (\ref{mydecay}) for $t=[T]$ and multiplying both sides by $\frac{1}{CK \eta_g\eta_l T}$ with $(\frac{1}{2} - 30K^2 L^2 \eta_l^2 -\frac{q}{N}30K\eta_g\eta_l ) >C>0$ if $\eta_l < \min(\frac{1}{KL},\frac{N}{150Kq})$, we have
	\begin{equation}
		\begin{split}
			&\frac{1}{T}\sum_{t=1}^{T}\E [\|\nabla F(\bm{w}^{t+1})\|^2 ] = \frac{1}{T}\sum_{t=1}^{T}\E [\|\nabla F(\tw^{t+1})\|^2 ] \\
			& \leq \frac{F(\tw^t) - F(\tw^{t+1})}{CK\eta_g\eta_l T} + \frac{1}{C} (10KL^4 \eta_l^2 \rho^2 \sigma_l^2 + 30K^2 L^2 \eta_l^2 \gamma \\&~~~~+ 120 K^4 L^6 \eta_l^6 \rho^2 + 16K^3 \eta_l^4 L^6 \rho^2 + \frac{ \eta_g\eta_l L^3 \rho^2}{N}\sigma_l^2 )
            \\&~~~~+\frac{q\eta_g}{CN}(10\eta_l L^2\rho^2\sigma^2_g + 30K\eta_l\gamma + 24K^2\eta_l^3L^4\rho^2)\\
			& \leq \frac{F(\tw^0 ) - F(\tw^*)}{CK \eta_g\eta_l T} + \frac{1}{C}(10KL^4 \eta_l^2 \rho^2 \sigma_l^2 + 30K^2 L^2 \eta_l^2 \gamma \\&~~~~+ 120 K^4 L^6 \eta_l^6 \rho^2 + 16K^3 \eta_l^4 L^6 \rho^2 + \frac{ \eta_g\eta_l L^3 \rho^2}{N}\sigma_l^2 )
            \\&~~~~+\frac{q\eta_g}{CN}(10\eta_l L^2\rho^2\sigma^2_l + 30K\eta_l\gamma + 24K^2\eta_l^3L^4\rho^2),
		\end{split}
	\end{equation}
	where the second inequality uses $F(\tw^{t+1} ) \geq F(\tw^*) $ and $F(\tw^0 ) \geq F(\tw^t )$. Choosing the learning rates $\eta_l = \frac{1}{\sqrt{T}KL},\eta_g=\sqrt{KN}$ and perturbation amplitude $\rho$ proportional to the learning rate, e.g., $\rho = \frac{1}{\sqrt{T}}$, we have
	\begin{equation}
    \begin{split}
        \frac{1}{T}&\sum_{t=1}^{T}\E [\|\nabla F(\bm{w}^{t+1})\|^2] \\&= \mathcal{O}\bigg(\frac{F^*L}{\sqrt{TKN}} + \frac{\gamma}{T} + \frac{L^2 \sigma_l^2}{T^2 K} + \frac{L^2 \sigma_l^2}{T^{3/2}\sqrt{KN}} +  \frac{1}{T^{4}K^2} \\
        &~~~~~+ \frac{L^2}{T^3 K} + \frac{qL\sigma_l^2}{T^{3/2}\sqrt{KN}}+\frac{q\gamma\sqrt{K}}{\sqrt{TN}L}+\frac{qL}{\sqrt{KN}T^{5/2}}\bigg).
    \end{split}
	\end{equation}
    After omitting the higher order, we have
	\begin{equation}
    \begin{split}
    		\frac{1}{T}&\sum_{t=1}^{T}\E [\|\nabla F(\bm{w}^{t+1})\|^2] \\&= \mathcal{O}\bigg(\frac{F^*L}{\sqrt{TKN}} + \frac{\gamma}{T} + \frac{L^2 \sigma_l^2}{T^{3/2}\sqrt{KN}} +\frac{q\sqrt{K}\gamma}{\sqrt{TN}L}\\&~~~~~+ \frac{qL\sigma_l^2}{T^{3/2}\sqrt{KN}} \bigg) .    
    \end{split}
	\end{equation}
	
	This completes the proof. \hfill~$\Box$

\subsection{Proof of Theorem \ref{partial_ca}}
For the partial participation case, $S< N$. 

\begin{Lemma}(\cite{qu2022generalized})
		For the partial client participation, we can bound $\E_t [\|\Delta^t \|^2 ]$ as follows:
		\begin{equation}
        \begin{split}
			\E_t [\|\Delta^t \|^2 ] &\leq \frac{K\eta_l^2 L^2 \rho^2}{S}\sigma_l^2 + \frac{S}{N}\sum_i \bigg\|\sum_{j=1}^{K-1}\nabla F_i (\tw_{i,k}^t )\bigg\|^2 \\&+ \frac{S(S-1)}{N^2}\bigg\|\sum_{j=0}^{K-1}\nabla F_i (\tw_{i,j}^t )\bigg\|^2 .         
        \end{split}
		\end{equation}
	\end{Lemma}


\begin{Lemma}\label{Lemma:tt}(\cite{qu2022generalized})
		For $\E [\|\sum_{k}\nabla F_i (\tw_{i,k})\|^2 ]$, we have
		\begin{equation}
			\begin{split}
				\sum_i \E \bigg[\bigg\|\sum_{k}\nabla F_i (\tw_{i,k}) \bigg\|^2 \bigg] & \leq 30NK^2 L^2 \eta_l^2 (2L^2 \rho^2 \sigma_l^2 \\&~~~~~+ 6K \gamma + 6K\|\nabla F(\tw)\|^2 )\\ + 144K^4 L^6 \eta_l^4 \rho^2 
				 &+ 12NK^4 L^2 \eta_l^2 \rho^2 + 3NK^2 \gamma \\&~~~~~+ 3NK^2 \|\nabla F(\tw)\|^2,
			\end{split}
		\end{equation}
		where the expectation is w.r.t the stochasticity of the algorithm.
	\end{Lemma}

Based on the lemmas established above, we present the proof of Theorem \ref{partial_ca}.

\noindent\textit{Proof.} 
    
	\begin{equation}
		\begin{split}
			& \E[F(\tw^{t+1} )] \overset{\text{(a)}}{\leq} \E[F(\tbw_K^{t})] + \E_t [\langle \nabla F(\tbw_K^{t} ), \tw^{t+1}-\tbw_K^{t} \rangle]\\&~~~~+ \frac{L}{2}\mathbb{E}\Vert \tw^{t+1} - \tbw_K^{t}\Vert_2^2\\
			& \overset{\text{(b)}}{\leq} F(\tw^t ) - \frac{K \eta_g\eta_l}{2} \|\nabla F(\tw^t )\|^2 + K \eta_g\eta_l L^2 \mathcal{E}_{w} + K \eta_g\eta_l L^2 \mathcal{E}_{\epsilon}\\&~~~~ - \frac{\eta_g \eta_l}{2KN}\E_t \bigg[\bigg\|\sum_{i,k}\nabla F_i (\tw_{i,k}^t )\bigg\|^2 \bigg] \\
            &~~~~+ \frac{L\eta_g^2}{2} \E_t [\|\Delta^t \|^2 ]+ \frac{\eta_g^2}{S^2}\sum_i q\mathbb{E}\Vert \tw^r_{i,K}-\tbw_K^t \Vert_2^2  \\
			& \overset{\text{(c)}}{\leq} F(\tw^t ) - \frac{K \eta_g\eta_l}{2} \|\nabla F(\tw^t )\|^2 + K \eta_g\eta_l L^2 \mathcal{E}_{w} + K \eta_g\eta_l L^2 \mathcal{E}_{\epsilon} \\&~~~~+ \frac{K \eta_g^2\eta_l^2 L^3 \rho^2}{2S}\sigma_l^2  - \frac{ \eta_g\eta_l}{2KN}\E_t \bigg[\bigg\|\sum_{i,k}\nabla F_i (\tw_{i,k}^t )\bigg\|^2 \bigg] \\&~~~~+  \frac{ L\eta_g^2\eta_l^2}{2NS}\sum_i \bigg\|\sum_{j=1}^{K-1}\nabla F_i (\tw_{i,k}^t )\bigg\|^2 \\
            &~~~~+ \frac{\eta_g^2\eta_l^2 L(S-1)}{2SN^2}\bigg\|\sum_{j=0}^{K-1}\nabla F_i (\tw_{i,j}^t )\bigg\|^2 \\
            &~~~~+\frac{q\eta_g^2}{S} (5K\eta_l^2 (2L^2 \rho^2 \sigma_l^2 + 6K\gamma + 6K\|\nabla F(\tw)\|^2 ) \\&~~~~+ 24K^3 \eta_l^4 L^4 \rho^2)          \\
                                                \end{split}\nonumber
    	\end{equation}
        	\begin{equation}
    		\begin{split}
			& \overset{\text{(d)}}{\leq} F(\tw^t ) - \frac{K \eta_g\eta_l}{2} \|\nabla F(\tw^t )\|^2 + K \eta_g\eta_l L^2 (\mathcal{E}_{w} +\mathcal{E}_{\epsilon}) \\&~~~~+ \frac{K \eta_g^2\eta_l^2 L^3 \rho^2}{2S}\sigma_l^2 \\& +\frac{L \eta_g^2 \eta_l^2}{2NS}\sum_{i} \|\sum_k \nabla F_i (\tw_{i,k}^t )\|^2 
            +\frac{q\eta_g^2}{S} (5K\eta_l^2 (2L^2 \rho^2 \sigma_l^2 + 6K\gamma \\&~~~~+ 6K\|\nabla F(\tw)\|^2 ) + 24K^3 \eta_l^4 L^4 \rho^2)\\
			& \overset{\text{(e)}}{\leq} F(\tw^t ) - K \eta_g\eta_l \bigg(\frac{1}{2} - 30K^2 L^2 \eta_l^2 - \frac{L \eta_g\eta_l}{2S}(3K \\&~~~~+ 180K^3 L^2 \eta_l^2 )-\frac{q}{S}30K\eta_g\eta_l \bigg) \|\nabla F(\tw^t )\|^2 \\
			&~~~~ + K \eta_g\eta_l \bigg(10KL^4 \eta_l^2 \rho^2 \sigma_l^2 + 30K^2 L^2 \eta_l^2 \gamma
			+ 120 K^4 L^6 \eta_l^6 \rho^2 \\&~~~~+ 16K^3 \eta_l^4 L^6 \rho^2 +\frac{L^3 \eta_g\eta_l \rho^2}{2S}\sigma_l^2 \bigg)\\&~~~~
			 + \frac{K \eta_g^2\eta_l^2}{S} (30KL^5 \eta_l^2 \rho^2 \sigma_l^2 + 60K^2 L^3 \eta_l^2 \gamma + 72K^3 L^7 \eta_l^4 \rho^2 \\&~~~~+ 6K^3 L^3 \eta_l^2 \rho^2 + 2KL\gamma )
            +\frac{q\eta_g^2}{S} (5K\eta_l^2 (2L^2 \rho^2 \sigma_l^2 +6K\gamma ) \\&~~~~+ 24K^3 \eta_l^4 L^4 \rho^2)\\
			& \overset{\text{(f)}}{\leq} F(\tw^t ) - CK\eta_g \eta_l \|\nabla F(\tw^t )\|^2 \\
			&~~~~ + K\eta_g \eta_l \bigg(10KL^4 \eta_l^2 \rho^2 \sigma_l^2 + 90K^2 L^2 \gamma 
			+ 120 K^4 L^6 \eta_l^6 \rho^2 \\&~~~~+ 16K^3 \eta_l^4 L^6 \rho^2 +\frac{L^3 \eta_g\eta_l \rho^2}{2S}\sigma_l^2 \bigg)
			 \\&~~~~+ \frac{K \eta_g^2\eta_l^2}{S} (30KL^5 \eta_l^2 \rho^2 \sigma_l^2 + 60K^2 L^3 \gamma + 72K^3 L^7 \eta_l^4 \rho^2 \\&~~~~+ 6K^3 L^3 \eta_l^2 \rho^2 + 2KL\gamma )\\&~~~~
            +\frac{q\eta_g^2}{S} (5K\eta_l^2 (2L^2 \rho^2 \sigma_l^2 +6K\gamma )+ 24K^3 \eta_l^4 L^4 \rho^2) ,
		\end{split}
	\end{equation}
	where (a) is from Lemma~\ref{Lemma:full_init}; (b) is from Lemma~\ref{Lemma:decentsam} and Assumption \ref{ass:quant_var}; (c) is from \ref{Lemma:d2} and Lemma~\ref{Lemma:xdrift}; (d) is based on taking the expectation of $t$-th round and if the learning rates satisfy that $KL \eta_g\eta_l \leq \frac{S-1}{S}$; (e) is from Lemmas~\ref{Lemma:deltadrift}, \ref{Lemma:xdrift} and \ref{Lemma:tt} and (f) holds because there exists a constant $C > 0$ satisfying $(\frac{1}{2} - 30K^2 L^2 \eta_l^2 - \frac{L\eta_g \eta_l}{2S}(3K + 180K^3 L^2 \eta_l^2 )-\frac{q}{S}30K\eta_g\eta_l)>C>0$.
	
	Summing the above result for $t = [T]$ and multiplying both sides by $\frac{1}{CK \eta_g\eta_l T}$, we have
	\begin{equation}
		\begin{split}
			& \frac{1}{T}\sum_{t=1}^{T}\E [\|\nabla F(\bm{w}^{t+1})\|^2] \leq \frac{F(\tw^t ) - F(\tw^{t+1})}{CK \eta_g\eta_l T}\\
			& + \frac{1}{C}\bigg(10KL^4 \eta_l^2 \rho^2 \sigma_l^2 + 30K^2 L^2 \eta_l^2 \gamma
			+ 120 K^4 L^6 \eta_l^6 \rho^2 \\&+ 16K^3 \eta_l^4 L^6 \rho^2 +\frac{L^3 \eta_g \eta_l \rho^2}{2S}\sigma_l^2 \\& 
			+ \frac{ \eta_g\eta_l}{S}(30KL^5 \eta_l^2 \rho^2 \sigma_l^2 + 60K^2 L^3 \gamma \\
            &+ 72K^3 L^7 \eta_l^4 \rho^2 + 6K^3 L^3 \eta_l^2 \rho^2 + 2KL\gamma )\bigg)\\
            &+\frac{q\eta_g}{CS}(10\eta_l L^2\rho^2\sigma^2_g + 30K\eta_l\gamma + 24K^2\eta_l^3L^4\rho^2)\\
                                                \end{split}\nonumber
    	\end{equation}
        	\begin{equation}
    		\begin{split}
			& \leq \frac{F^*}{CK \eta_l T}  + \frac{1}{C}\bigg(10KL^4 \eta_l^2 \rho^2 \sigma_l^2 + 30K^2 L^2 \eta_l^2 \gamma 
			\\&+ 120 K^4 L^6 \eta_l^6 \rho^2 \\&+ 16K^3 \eta_l^4 L^6 \rho^2  +\frac{L^3  \eta_g\eta_l \rho^2}{2S}\sigma_l^2 
			+ \frac{\eta_g \eta_l}{S}(30KL^5 \eta_l^2 \rho^2 \sigma_l^2 \\&+ 60K^2 L^3 \eta_l^2 \gamma + 72K^3 L^7 \eta_l^4 \rho^2 + 6K^3 L^3 \eta_l^2 \rho^2 + 2KL\gamma)\bigg)\\
            &+\frac{q\eta_g}{CS}(10\eta_l L^2\rho^2\sigma^2_l + 30K\eta_l\gamma + 24K^2\eta_l^3L^4\rho^2), 
		\end{split}
	\end{equation}
	where the second inequality uses $F^*= F(\tw^0)-F(\tw^*) \geq F(\tw^t )-F(\tw^{t+1})$. If we choose the learning rates $\eta_l = \frac{1}{\sqrt{T}KL}$, $\eta_g = \sqrt{KS}$ and perturbation amplitude $\rho$ proportional to the learning rate, e.g., $\rho = \frac{1}{\sqrt{T}}$, we have
	\begin{equation}
		\begin{split}
			\frac{1}{T}&\sum_{t=1}^{T}\E [\|\nabla F(\bm{w}^{t+1})\|^2] \\& = \mathcal{O}\bigg(\frac{F^*L}{\sqrt{TKS}}+ \frac{\gamma}{T}  +\frac{L^2\sigma_l^2}{T^2K} + \frac{1}{T^4K^2}+\frac{L^2}{T^3K}\\&~~~~+\frac{L^2\sigma_l^2}{T^{3/2}SK}+\frac{L^2\sigma_l^2}{T^{5/2}\sqrt{SK}}+\frac{\gamma}{T^{3/2}\sqrt{SK}}\\&~~~~+\frac{L^2}{T^{7/2}K^{3/2}S^{1/2}}+\frac{\sqrt{K}}{T^{5/2}\sqrt{S}}+\frac{\sqrt{K}\gamma}{\sqrt{TS}}
            \\&~~~~+ \frac{qL\sigma_l^2}{T^{3/2}\sqrt{KS}}+\frac{q\sqrt{K}\gamma}{\sqrt{TS}L}+\frac{qL}{\sqrt{KS}T^{5/2}}\bigg),
		\end{split}
	\end{equation}
	After omitting the larger order, we have
	\begin{equation}
    \begin{split}
      		\frac{1}{T}\sum_{t=1}^{T}&\E [\|\nabla F(\bm{w}^{t+1})\|^2] = \mathcal{O}\bigg(\frac{F^*L}{\sqrt{TKS}}+ \frac{\sqrt{K}\gamma}{\sqrt{TS}} \\&+ \frac{L^2 \sigma_l^2}{T^{3/2}K} + \frac{q\sqrt{K}\gamma}{\sqrt{TS}L}+\frac{qL\sigma_l^2}{T^{3/2}\sqrt{KS}}\bigg).  
    \end{split}
	\end{equation}
	This completes the proof. \hfill~$\Box$

\end{document}